\documentclass[mnsc,nonblindrev]{informs3-meng}

\OneAndAHalfSpacedXI %

\usepackage{natbib}
 \bibpunct[, ]{(}{)}{,}{a}{}{,}%
 \def\bibfont{\small}%

\usepackage{amsmath} 
\usepackage{bbm}
\usepackage{amssymb}
\usepackage{mathtools}
\usepackage{soul}
\usepackage{dsfont}
\usepackage{color}
\usepackage{bm}
\usepackage{longtable}
\usepackage{threeparttable}
\usepackage{enumitem}
\usepackage[skip=0pt,font=scriptsize]{subcaption}
\usepackage{mathtools, nccmath}
\usepackage[skip=2pt,font=scriptsize]{caption}
\newcommand{\bbR}{\mathbb{R}}
\newcommand{\calX}{\mathcal{X}}
\newcommand{\calH}{\mathcal{H}}
\newcommand{\calB}{\mathcal{B}}

\newcommand{\calE}{\mathcal{E}}

\newcommand{\calD}{\mathcal{D}}

\newcommand{\calP}{\mathcal{P}}
\newcommand{\calQ}{\mathcal{Q}}

\newcommand{\bbE}{\mathbb{E}}
\newcommand{\bbP}{\mathbb{P}}

\newcommand{\p}{p}

\TheoremsNumberedThrough     %
\ECRepeatTheorems

\EquationsNumberedThrough    %

\MANUSCRIPTNO{}

\begin{document}

\RUNAUTHOR{Qi, Grigas, Shen}

\RUNTITLE{Integrated Conditional Estimation-Optimization}

\TITLE{Integrated Conditional Estimation-Optimization}

\ARTICLEAUTHORS{%
\AUTHOR{Meng Qi}
\AFF{SC Johnson College of Business, Cornell University, Ithaca, NY, 14850, \\ \EMAIL{mq56@cornell.edu}}
\AUTHOR{Paul Grigas}
\AFF{Department of Industrial Engineering and Operations Research, UC Berkeley, Berkeley, CA, 94720, \\ \EMAIL{pgrigas@berkeley.edu}}
\AUTHOR{Max Shen}
\AFF{%
Faculty of Engineering \& Faculty of Business and Economics, University of Hong Kong, China, \\ \EMAIL{maxshen@hku.hk}}
} 

\ABSTRACT{%
Many real-world optimization problems involve uncertain parameters with probability distributions that can be estimated using contextual feature information. 
 In contrast to the standard approach of first estimating the distribution of uncertain parameters and then optimizing the objective based on the estimation, we propose an \textit{integrated conditional estimation-optimization} (ICEO) framework that estimates the underlying conditional distribution of the random parameter while considering the structure of the optimization problem. We directly model the relationship between the conditional distribution of the random parameter and the contextual features, and then estimate the probabilistic model with an objective that aligns with the downstream optimization problem. %
We show that our ICEO approach is asymptotically consistent under moderate regularity conditions and further provide finite performance guarantees in the form of generalization bounds.
Computationally, performing estimation with the ICEO approach is a non-convex and often non-differentiable optimization problem. We propose a general methodology for approximating the potentially non-differentiable mapping from estimated conditional distribution to optimal decision by a differentiable function, which greatly improves the performance of gradient-based algorithms applied to the non-convex problem.
We also provide a polynomial optimization solution approach in the semi-algebraic case. Numerical experiments are also conducted to show the empirical success of our approach in different situations including with limited data samples and model mismatches.
}%


\maketitle

\section{Introduction}
\label{sec: intro}
Two fundamental aspects of decision-making under uncertainty are estimation and optimization. Classically these two aspects are treated separately, with statistical and/or machine learning methodologies used to estimate the distributions of uncertain parameters based on data, resulting in a stochastic optimization problem to be solved for making a decision. In recent years, researchers and practitioners have increasingly recognized the significance of considering estimation and optimization in tandem \citep{bertsimas2020predictive, kao2009directed, donti2017task, elmachtoub2022smart}. Another salient feature of modern decision-making under uncertainty is the presence of \emph{contextual} information, usually in the form of features/covariates, that can be leveraged to improve the estimation of the uncertain parameters. For example, contextual information such as temporal information, the presence of promotions, and economic indicators can be leveraged to refine the estimation of uncertain demand for products. The refined demand distribution estimates would then be used for making inventory and supply chain decisions through optimization models. \emph{Contextual stochastic optimization (CSO)} has recently emerged as a general paradigm describing this situation, with applications in supply chain management, finance, transportation, energy systems, and many other areas. 

In this work, we consider the CSO problem in a data-driven setting where one has available historical data consisting of realizations of the uncertain parameters paired with contextual feature information. 
As mentioned, the classical method of solving CSO given data is a two-step procedure, where in the first step either a point prediction of the parameter or an estimation of its distribution is built based on data. (Although the phrases ``prediction" and ``estimation" are often synonymous or not clearly distinguished in the literature, herein we specifically let ``prediction" denote point predictions of the random parameter and let ``estimation" refer to any methodology, either parametric or non-parametric, for estimating the conditional distribution of the random parameter given the context.) 
Modern machine learning techniques are often utilized in the first step to provide more granular results, and these models are usually fit based on statistical objectives such as measures of prediction error or likelihood. 
Then in the second step, given the prediction or estimation, an optimization problem is solved.
A major drawback of these standard predict-then-optimize (PTO) and estimate-then-optimize (ETO) approaches is that they do not consider the decision error -- the cost with respect to the downstream optimization problem due to an imperfect prediction -- when fitting a statistical model. 

We propose an integrated conditional estimation-optimization (ICEO) approach that estimates the underlying conditional distribution of the random parameter based on minimizing the ultimate decision error. 
We propose a highly flexible framework that models the conditional distribution using a hypothesis class and applies ideas from statistical learning to do estimation. As compared to existing approaches, our approach uses a generic learning framework based on specifying a hypothesis class and applies to a broad class of convex contextual stochastic optimization problems with uncertainty in the objective. We study the statistical and computational properties of our approach. In particular, we prove asymptotic consistency in terms of risks, decisions and hypotheses. Asymptotic consistency is highly desired for data-driven methods because it guarantees that, as the amount of data increases, our solutions and estimated models converge to their optimums given full information of the true distribution of contextual features and uncertain parameters. We prove asymptotic consistency in terms of the ICEO risk, induced decisions, and the learned hypothesis. We also provide generalization bounds to quantify the out-of-sample performance when data is limited to a finite sample.

In general, there are fundamental differences between the cases of linear and nonlinear objective functions in CSO problems and the two problem classes require completely different solution methods. As pointed out in the prior literature (for example, \cite{bertsimas2020predictive}, \cite{elmachtoub2022smart}, \cite{sadana2023survey} and \cite{qi2022integrating}), the linear objective assumption implies that point predictions are sufficient to address the CSO problem. %
However, in the general nonlinear case, one needs to utilize a distributional estimation of the conditional distribution. As such, ICEO seeks a distributional estimation that directly models the relationship between the conditional distribution of the random parameter and the contextual features and then estimates the probabilistic model with a training objective that aligns with the downstream optimization problem.
Our methodological and corresponding theoretical contributions are entirely novel and significant.

It is worthwhile to point out that, in the ICEO method, we adopt a strongly convex decision regularization function inside the optimization oracle in order to stabilize the decision and induce uniqueness. Due to the interplay between the downstream optimization oracle, decision regularization, and the model, the overall training objective is more complicated than typical prediction error losses and as such standard uniform convergence and consistency do not directly apply even with i.i.d. samples.
Instead, such results can only be obtained after establishing the convergence and Lipschitzness of the regularized oracle. Similarly, finite sample bounds for ICEO are generalization bounds based on multi-variate Rademacher complexity. However, as ICEO adopts a more complicated training objective that incorporates the regularized oracle to take the downstream problem into account, ICEO deviates from standard learning procedures that use simple loss functions. Tackling this difficulty also relies on establishing a Lipschitz property of the regularized oracle, which further translates to the training objective.

In terms of computation, the core training problem of the ICEO framework is non-convex and even non-differentiable in many cases. In fact, due to the presence of constraints in the downstream problem, it is often the case that the optimal decision oracle has a piece-wise constant shape, which leads to poor local minima that are very hard to escape when applying gradient-based methods.
For these reasons, we propose two computational approaches:  {\em (i)} a highly practical approach that involves approximating the regularized optimal solution oracle with a smooth function and then applying gradient algorithms, and {\em (ii)} a polynomial optimization approach when the downstream problem has a semi-algebraic objective and we approximate the optimal solution oracle with a polynomial function.

Our key contributions are summarized as follows:
\begin{enumerate}
    \item We propose the ICEO framework, wherein we directly estimate the underlying conditional distribution of uncertain parameters given contextual information using a hypothesis class. In contrast to two-step ETO methods, we learn the conditional distribution in a way that integrates with the downstream optimization goal. ICEO offers more flexibility compared to most existing related approaches. To the best of our knowledge, among all integrated approaches for general classes of nonlinear CSO problems with convex objective functions, ICEO is the first that provides both asymptotic and finite sample performance guarantees.
    
    \item We prove the asymptotic consistency of the ICEO method when the model is specified correctly (Theorem \ref{thm: consistency}). More specifically, we show the consistency of ICEO risk, ICEO decisions, and ICEO hypothesis when the hypothesis class contains the correct conditional distribution function.
    To show asymptotic consistency, standard statistical learning approaches do not apply because of the presence of the regularized oracle.
    Tackling these difficulties requires establishing certain properties of the regularized oracle such as uniform convergence and uniform Lipschitzness. %
    \item To quantify the out-of-sample performance with finite samples, we provide generalization bounds for the ICEO method based on the multi-variate Rademacher complexity of the hypothesis class used to learn the conditional distribution (Theorem \ref{thm:finite-sample-bd}). Again, standard statistical learning results do not directly apply since taking the downstream optimization into account leads to a more complicated loss function than the typical (e.g., MSE) losses. In particular, our generalization bound requires first establishing the Lipschitz property of the regularized optimal decision oracle (Proposition \ref{prop:w-lip}).Furthermore, our generalization bound intuitively indicates that ICEO can be more flexible/expressive than policy optimization methods while still achieving a comparable sample (Rademacher) complexity.
    
    \item The ICEO training problem is non-convex and non-differentiable. Non-differentiability poses a serious concern when applying gradient-based algorithms, like (stochastic) gradient descent, to solve the ICEO training problem as the presence of constraints can lead to local minima that are hard to escape (visually illustrated in Figure \ref{fig:fig}). To address this issue, we approximate the oracle using differentiable function classes with a guaranteed approximation error (Proposition \ref{prop: bern-approx-error}, Proposition \ref{prop: krr-approx-error}). We then provide corresponding generalization bounds when training ICEO method using the approximated oracle (Theorem \ref{thm: gen-bd-approx}). In addition, for the case where the nominal optimization problem is semi-algebraic, we propose an exact solution algorithm (Proposition \ref{prop: simplex-constr}).
\end{enumerate}
The remainder of this paper is organized as follows. %
In Section \ref{sec: literature}, we review related methods in literature. The details of our proposed ICEO framework are introduced in Section \ref{sec: framework}. In Section \ref{sec: general-guarantee}, we provide performance guarantees in terms of asymptotic consistency and generalization bounds. In Section \ref{sec: computational}, we discuss the main difficulties in solving the ICEO formulation and provide solution methods. Empirical performances of the ICEO method is demonstrated in Section \ref{sec:experiments}. Moreover, Appendix \ref{sec: consistency-proof-appendix} provides supplementary materials to support the asymptotic consistency result in Section \ref{sec:consistency}. \ref{sec: proofs} presents supplementary lemmas and proofs for Sections \ref{sec:generalization} and \ref{sec: computational}. Appendix \ref{sec: appendix-jumping} presents conditions that ensure an ``automatic crossover'' behavior of the regularized oracle, which further supports the assumptions made in Section \ref{sec:consistency}. Appendix \ref{sec: approximation-polynomials} provides two detailed examples of using polynomial functions to approximate the optimal solution mapping. When both the oracle approximation and objective functions are polynomial functions, the ICEO problem can be reformulated as a semi-algebraic problem thus can be solved efficiently. Appendix \ref{appendix: exp} demonstrates supplementary materials for Section \ref{sec:experiments}. %
In Appendix \ref{sec: policy-opt}, we provide more intuition of the advantages of ICEO by comparing it to policy optimization approaches.

\subsection{Relevant Literature}
\label{sec: literature}

The fusion of prediction models based on data and the optimization problems has become more and more widespread in recent years. In the remainder of this section, we will discuss existing works related to this topic and contrast them with our proposed ICEO approach. 

The first stream of research focuses on providing a prescriptive solution by approximating the conditional distribution of the random parameter given a feature vector with the help of various machine learning tools. In one of the first works along these lines, \cite{hannah2010nonparametric} use nonparametric methods to estimate the density function conditioned on state variables to solve convex stochastic optimization problems. They consider weights on the empirical distribution based on kernels and the Dirichlet process. \cite{bertsimas2020predictive} also proposed prescriptive models that approximate the conditional distribution with a weighted empirical distribution of the uncertainty. The weights can be achieved based on multiple machine learning models, including k-nearest neighbors (KNN), kernel methods, tree-based methods, etc. \cite{kallus2020stochastic} consider using a (non-parametric) random forests estimator of the conditional distribution in a way that is trained with respect to the cost of the downstream optimization task, akin to the ICEO approach.  A later work \cite{bertsimas2019predictions} investigates these prescriptive methods in the multi-period problem setting. \cite{bertsimas2019optimal} follows the same idea and propose a tree-based algorithm that balances the optimality of the prescription and accuracy of the prediction. \cite{kallus2020stochastic} consider a random forest model for the prescriptive solution. In contrast to the standard way of splitting the feature space, the authors consider the down stream optimization quality while constructing the partitions. Different from \cite{kallus2020stochastic} the ICEO approach directly models the underlying conditional distribution using a hypothesis class $\calH$. Thus, while \cite{kallus2020stochastic} only provide asymptotic consistency results, we are able to prove both asymptotic consistency and generalization bounds for a wide variety of hypothesis classes.

\cite{ho2019data} considers the regularized Nadaraya-Watson approach and establishes performance guarantees using moderate deviations theory. We refer to \cite{qi2022integrating} for a tutorial on these methods. 

Another stream of related work investigates adjusting the loss function to meet the ultimate optimization goal while training the machine learning models to predict the random parameters. \cite{ban2019big} investigates the Newsvendor problem, which is inherently equivalent to a quantile prediction problem. The authors learn the feature-to-decision mapping from data by adopting a loss function that characterizes the newsvendor inventory cost and is equivalent to the quantile loss function. %
\cite{elmachtoub2022smart} consider the case when the downstream optimization problem has a linear objective. The authors propose a ``smart predict-then-optimize" (SPO) framework with a tractable convex surrogate loss function (SPO+) to integrate the ultimate optimization problem structure. They prove Fisher consistency of SPO+ and demonstrate its strong numerical performance on different problem classes. \cite{balghiti2019generalization} later provides a finite-sample performance guarantee of the SPO loss in the form of generalization bounds. Recently, \cite{liu2021risk} have strengthened the consistency of SPO+ by providing risk guarantees and a calibration analysis in the polyhedral and strongly convex cases. 
\cite{elmachtoub2020decision} propose a method to train decision trees using the SPO loss and demonstrate strong numerical performance and improved model complexity over standard decision tree methods (e.g., CART) that minimize prediction error.

Other existing studies aim to learn the task-based end-to-end learning models with differentiable optimization layers.  
\cite{donti2017task} consider a general setting where the optimization stage involves a convex optimization problem and
adopt the objective in the optimization stage as the loss function to achieve an end-to-end training for the machine learning models. The main issue in such end-to-end learning models is to address the non-differentiability of the optimal solution mapping (the mapping from a contextual feature vector to the optimal decision). \cite{amos2017optnet} introduce the differentiable optimization layers for the end-to-end training approaches and propose a method of approximating the gradient of the optimal solution mapping by the solution of a group of equations representing the KKT conditions. \cite{agrawal2019differentiable} further provide a method to convert convex programs to the canonical forms that can be implemented at the optimization layer and implemented their grammar in CVXPY for ease of use.
\cite{wilder2019melding} and \cite{wilder2019end} further consider more difficult combinatorial problems. They propose end-to-end models that map from the graph structure to a feasible solution and train them with the quality of the solution. \cite{wilder2019melding} consider continuous relaxations of the discrete problem to propagate gradients through the optimization procedure. \cite{mandi2020interior} consider mixed integer linear programs and consider a homogeneous self-dual formulation of the LP and show that the gradients are related to an interior point step. \cite{berthet2020learning} instead consider stochastically perturbed optimizers to evaluate the gradients required for back-propagation. \cite{mandi2020smart, ferber2020mipaal,poganvcic2019differentiation} also discuss how to approximate the gradients when training end-to-end models for combinatorial problems. As our work focuses on convex optimization problems, we skip the details and refer to \cite{kotary2021end} for a detailed survey. Although demonstrated to be competitive in numerical experiments, these end-to-end learning models based on optimization layers and their extensions to combinatorial cases lack strong performance guarantees in theory. Moreover, learning feature-to-decision mapping lacks flexibility in the way that it handles constraints. Indeed, constraints restrict the hypothesis class that can be used to learn the data-to-decision mapping. In contrast, our ICEO framework learns the conditional distribution and uses the optimal solution mapping to obtain the decision, which is more flexible in handling constraints.

Other related works include \cite{kao2009directed} which investigates the case of model mis-specification when features are not perfect. This work proposes a method of directed regression which combines the merits of regression and empirical optimization. Later \cite{kao2012directed} extended this setting to a directed time series regression.
\cite{horisk}, which investigates the relationship between the prediction part to the performance of the optimization part, mainly in the case of the least squares loss function. A recent work \cite{poursoltani2023robust} investigates the coefficient of prescriptiveness which can be used to quantify the performance of different CSO approaches and how to directly optimize it. \cite{bennouna2021learning} took a different perspective on data-driven solutions and investigated the optimal formulation that guarantees a certain level of out-of-sample performance. 

There are also a number of decision-focused learning contributions that address practically impactful applications. \cite{chung2022decision} investigate the problem of allocating limited supply in health supply chains and proposed a decision-aware learning method that uses the decision cost to inform training. \cite{qi2020practical} focus on a multi-period inventory management problem with random demand and lead times, and provide a practical end-to-end learning framework empowered by deep learning models. The authors demonstrate the empirical success of this approach in practice by conducting a field experiment in industry. \cite{cristian2022endto} proposes a neural network architecture that approximately solves linear programs in an end-to-end way. The authors also analyze applications of the proposed approach to a multi-warehouse inventory management problem with cross-fulfillment. \cite{chehrazi2010monotone} consider the problem of optimal debt settlement and consider an approach that combines estimating the objective function and the optimization problem.

Another stream of work investigates the performance of the classic separated two-step PTO and ETO approach. \cite{hu2022fast} demonstrate that, when the optimization problem has a linear objective and linear constraints, the two-step predict-then-optimize approach leads to faster convergence in terms of expected risk compared to the integrated policy optimization approach. \cite{elmachtoub2023estimate} take a stochastic dominance perspective and demonstrate that when the model class is well-specified, the predict-then-optimize approach outperforms the integrated approach in a strong sense.
More broadly, there is a rapidly growing literature regarding data-driven methods for CSO problems. We have mainly included a discussion of work closely related to our paper and refer readers to \cite{sadana2023survey} for a comprehensive survey of various approaches to address CSO as well the connection among previously stated streams of literature.

There are also other studies that explore seemingly similar but different problem settings as CSO and ICEO. For example, the joint estimation-optimization (JEO) model (\cite{jiang2013solution, ahmadi2014data,jiang2016solution, ho2019exploiting})
 The major difference between CSO and JEO is that, in the JEO model, there is no contextual information considered predictors of uncertainty. Besides, several JEO models focus on solving an online convex optimization problem in the optimization stage, while we consider a stochastic optimization problem. 
We would also like to point out the differences between CSO and the operational statistics method, in which the downstream optimization goal is considered in finding the optimal operational statistic (\cite{liyanage2005practical, chu2008solving, ramamurthy2012inventory}). We include contextual information in our problem setting which is not considered in the classic operational statistics literature. Moreover, we aim to learn the underlying conditional distribution rather than finding the best statistic. We also consider constraints in the downstream optimization problem.

Lastly, there are other works that explore data-driven solutions for stochastic optimization. For example, data-driven methods for robust optimization (\cite{bertsimas2018data, wang2023learning, bertsimas2018robust, hong2021learning}) and distributionally robust optimization (DRO) (\cite{delage2010distributionally,  blanchet2019robust, gao2023distributionally, wiesemann2014distributionally, van2021data}). These lines of research emphasize the robustness of data-driven solutions. Among the data-driven DRO literature, the stream of residual-based DRO is most relevant to our work (\cite{kannan2020residuals, qi2022distributionally}). Residual-based DRO considers the conditional distribution based on residuals. A similar approach involving residual-SAA has been studied by \cite{deng2022predictive,liu2022coupled,kannan2022data}.
Besides methodologies that solves the standard CSO, there are other works considering other settings different from the classical CSO problem, for example, the small-data large-scale regime \cite{gupta2021small, gupta2022debiasing}.  

\section{Contextual Stochastic Optimization and the ICEO Approach}
\label{sec: framework}
In this section, we review the basic ingredients of contextual stochastic optimization problems, which is a fundamental model for applying machine learning in many operational contexts, and we formally describe our ICEO approach.
We consider a convex CSO, which models a downstream decision-making task. The feasible region for the decision variable $w \in \bbR^d$, denoted by $S \subset \mathbb{R}^d$, is assumed to be known with certainty. We additionally assume that $S$ is a convex and compact set. Although the feasible region of our optimization task is known with certainty, the objective function $c(\cdot, \xi) : S \to \bbR$ is stochastic and depends on a random parameter $\xi$. We assume that for all values of $\xi$, $c(\cdot, \xi)$ is a convex function of $w$. While the precise value of $\xi$ is not known at the time when a decision must be made, we assume that the decision-maker observes an associated contextual feature vector $x \in \calX \subseteq \bbR^p$ (sometimes the components of $x$ are referred to as covariates) that can be used to learn information about the objective function. Let $\mathcal{D}$ denote the joint distribution of $x$ and $\xi$. 
Then, given an observed $x \in \bbR^p$, the decision maker's goal is to solve the contextual stochastic optimization problem:
 \begin{equation}
 \label{eq: sto-opt}
        \min_{w\in S} \ \mathbb{E}_\xi[c(w,\xi)|x],
    \end{equation}
where the expectation above is with respect to the \emph{conditional distribution} of $\xi$ given $x$. 

It is important to emphasize that the distribution $\mathcal{D}$, and hence the conditional distribution of $\xi$ given any $x$, is typically unavailable in practice. Instead, a data-driven approach to solving \eqref{eq: sto-opt} is much more viable. Indeed, one often has a training dataset $\{(x_i, \xi_i)\}_{i=1}^n$ consisting of historically observed pairs of feature vectors $x_i \in \calX$ and associated parameter values $\xi_i$. 

In this work, in order to directly model the conditional distribution, we consider the case where the random parameter $\xi$ has finite discrete support, i.e., $\xi \in \Xi := \{\tilde{z}_1, \tilde{z}_2, \dots, \tilde{z}_K\}$. 
Then, for any $x \in \calX$, the conditional distribution of $\xi$ given $x$ is characterized by a probability vector $p^\ast(x) \in \Delta_K$, where $\Delta_K := \{p \in \bbR^K : \sum_{k = 1}^K p_k = 1, p \geq 0\}$ denotes the $(K-1)$-dimensional unit simplex. That is, $p^\ast_k(x)$, the $k$-th component of $p^\ast(x)$, is defined by $p^\ast_k(x) = \bbP_\xi(\xi = \tilde{z}_k | x)$, for all $k = 1, \dots, K$. Using this notation as well as the shorthand notation $c_k(\cdot) := c(\cdot, \tilde{z}_k)$ for all $k = 1, \ldots, K$, problem \eqref{eq: sto-opt} can be equivalently written as
\begin{equation}
 \label{eq: sto-opt-2}
        \min_{w\in S} \ \mathbb{E}_\xi[c(w,\xi)|x] ~=~ \min_{w\in S} \ \sum_{k=1}^K p^\ast_k(x)c_k(w).
\end{equation}
Note that there might be multiple optimal solutions and we use the notation $W(p)$ to refer to the set of such optimal solutions, i.e., $W(p) := \argmin_{w\in S} \sum_{k=1}^K p_kc_k(w)$.
\subsection{ICEO Approach}
\label{sec: ICEO-approach}
Let us now describe the major ingredients of our ICEO approach, as well as the formulation of our ICEO training problem.
\paragraph{Hypothesis Class of Conditional Probability Estimators}
It is evident from the right side of \eqref{eq: sto-opt-2} that learning the conditional distribution $p^\ast(x)$ is the most critical part of our contextual stochastic optimization setting. We adopt standard ideas from learning theory to learn $p^\ast(x)$, whereby we employ a compact hypothesis class $\calH$ of conditional probability estimators. That is, $\calH$ is a compact set (e.g., with respect to the uniform norm) of functions $f : \calX \to \Delta_K$. The hypothesis class $\calH$ is the first major ingredient of our ICEO approach. Note that the constraint on the output of $f \in \calH$, namely $f(x) \in \Delta_K$ for all $x \in \calX$, is not standard in most learning problems but is necessitated by our setting. Fortunately, this constraint can be accommodated in a number of ways. A straightforward approach is to consider the softmax operator $\mathrm{soft} : \bbR^K \to \bbR^K$ defined by $\mathrm{soft}_k(v) = \frac{\exp(v_k)}{\sum_{j=1}^K \exp(v_j)}$ for $v\in \bbR^K$. Then, given \emph{any} hypothesis class $\tilde{\calH}$ of unconstrained functions $\tilde{f} : \calX \to \bbR^K$, we can define $\calH$ as the composition class $\mathrm{soft} \circ \tilde{\calH}$. Note that, due to the differentiability properties of the softmax operator, $\mathrm{soft} \circ \tilde{\calH}$ naturally inherits differentiability properties from $\tilde{\calH}$, which can be very useful from a computational perspective. For another example, consider $\calH$ defined by a decision tree partitioning algorithm. Then, for any given $x$, $f(x)$ can be constructed from the empirical distribution of $\xi$ restricted to the subset of the partition of the training data for which $x$ lies in. Finally, a third approach, which we expand upon in Section \ref{sec: semi-algebraic}, is to let $\mathcal{H}$ be a constrained linear hypothesis class whereby $\mathcal{H} = \{f : f(x) = Bx \in \Delta_K \ \text{for all } x \in \calX\}$. Depending on the structure of $\calX$, it may be possible to efficiently model the constraint $Bx \in \Delta_K \ \text{for all } x \in \calX$, and we discuss specific examples in Section \ref{sec: semi-algebraic}.
We would like to emphasize two points about our approach for estimating the conditional distribution using a hypothesis class $\calH$. First, by directly estimating the conditional probability our proposed method has more flexibility in handling constraints as compared to methods that learn a mapping $\pi$ directly from features $x$ to decisions $w$. In particular, compared to the policy learning approaches which learn a mapping from features to decisions requires that the output of the mapping $\pi$ be feasible in the region $S$, which may severely constrain the feasible set of $\pi$. More detailed intuitions can be found in Appendix \ref{sec: policy-opt}.
On the other hand, our approach of composing a user-specified hypothesis class $\calH$ with the regularized optimal solution mapping $w_\rho(\cdot)$ allows for a very general selection of $\calH$. 

\paragraph{Regularized Optimization Oracle.}
As mentioned previously, we assume that the functions $c_k(\cdot) = c(\cdot, \tilde{z}_k)$, for all $k = 1, \ldots, K$, are all convex functions of $w$ on the convex and compact feasible region $S$. %
Furthermore, we presume that we can additionally work with a \emph{decision regularization function} $\phi(\cdot) : S \to \bbR$, which is non-negative and strongly convex with respect to some norm $\|\cdot\|$ on $\bbR^d$. 
Given any $p \in \Delta_K$ and $\rho > 0$, define the regularized optimal solution mapping: 
\begin{align}
\label{program:origin-lower-level_regularized}
    w_\rho(p) := &\argmin_{w\in S} \sum_{k=1}^K p_kc_k(w) + \rho\phi(w).
\end{align}
Note that, due to the strong convexity of $\phi(\cdot)$, $w_\rho(p)$ is uniquely defined. Furthermore, we can show that $w_\rho(\cdot)$ is a continuous mapping as demonstrated in Lemma \ref{lemma: opt-cont-func}. 
These regularity properties are induced by the use of the regularization term $\phi(\cdot)$, which is crucial for developing our ICEO methodology and for providing associated theoretical guarantees. We further assume that $w_\rho(p)$ can be efficiently computed in practice for any $p \in \Delta_K$ and $\rho > 0$. Possible examples include using a commercial solver or utilizing a specialized algorithm that depends on the structure of the $c_k(\cdot)$ and $\phi(\cdot)$ functions.
Ideally, the function $\phi(\cdot)$ should be chosen so that the complexity of computing $w_\rho(p)$ is not greatly increased as compared to when $\rho = 0$.
Note that our performance guarantees developed in Section \ref{sec: general-guarantee} hold for any choice of $\phi(\cdot)$ that is $1$-strongly convex.

\paragraph{ICEO Methodology.}
We are now ready to describe our ICEO methodology and corresponding training problem, whereby we consider an integrated approach that estimates a hypothesis $f \in \calH$ in consideration of the downstream optimization goal. We presume that we have collected a training dataset $\{(x_i, \xi_i)\}_{i=1}^n$ consisting of historically observed pairs of feature vectors $x_i \in \calX$ and associated parameter values $\xi_i$. We also presume that the decision maker uses the regularized optimal solution oracle $w_\rho(\cdot)$ defined in \eqref{program:origin-lower-level_regularized}. 
We adopt the empirical risk minimization (ERM) principle with respect to the  in-sample cost induced by the regularized oracle:
\begin{align}
\label{program:origin}
\tag{ICEO-$\rho$}
    \min_{f\in \mathcal{H}, w_1, \ldots, w_n \in S} \quad & \frac{1}{n} \sum_{i = 1}^n c(w_i, \xi_i) \\ %
    \mathrm{s.t.} \quad & w_i =  w_{\rho}(f( x_i))\nonumber,
\end{align}
where $\rho > 0$ is a given value of the decision regularization parameter, which can be chosen with cross-validation for example.
Let $\hat{f} \in \calH$ denote a computed optimal solution of \eqref{program:origin}. Then, for any newly observed feature vector $x \in \calX$, the decision maker implements the decision $w_\rho(\hat{f}(x)) \in S$ formed by composing $w_\rho(\cdot)$ with $\hat{f}(\cdot)$. We remark that one may consider a variant of \ref{program:origin} that, in the objective function, additionally includes decision regularization terms $\rho\phi(w_i)$ for each sample. Although these decision regularization terms appear more aligned with \eqref{program:origin-lower-level_regularized}, and were included in an earlier version of this paper, they are in fact a purely optional component of our model and we choose to not include them here for simplicity. 
Indeed, our entire analysis also will carry through with or without the additional decision regularization terms. The only difference is that the rate of convergence in the finite-sample bound is slightly faster when these terms are excluded.

Let us contrast the ICEO approach with the PTO and ETO approaches. 
In the PTO approach, a machine learning model $\hat{g}_{\mathrm{PTO}}: \calX \to \Xi$ is built, using the training data, to predict the parameter $\xi$ based on the feature vector $x$. Then, given any new $x \in \calX$, the decision maker implements a decision from the optimal solution set $\arg\min_{w \in S}c(w, \hat{g}_{\mathrm{PTO}}(x))$.
Note that, as pointed out in Section \ref{sec: intro}, because of the nonlinearity of the objective function, a point estimate for a prediction of $\xi$ given $x$ does not provide enough information about the conditional distribution to produce a reasonable solution of \eqref{eq: sto-opt}. 
Different from PTO, the ETO approach learns a model $\hat{f}_{\mathrm{ETO}} : \calX \to \Delta_K$ for estimating the conditional distribution of $\xi$ given $x$.
Then, given any new $x \in \calX$, the decision maker implements a decision from the optimal solution set $W(\hat{f}_{\mathrm{ETO}}(x))$.
Thus, the ETO approach is more aligned with the ICEO approach. The main distinction is that the traditional ETO approach learns the model $\hat{f}_{\mathrm{ETO}}$ in a way that is completely oblivious to the downstream optimization task. For instance, given a hypothesis class $\calH$, the ETO approach might select the hypothesis by minimizing the empirical cross-entropy loss, defined for any $f \in \calH$ and any observed $(x, \xi = \tilde{z}_k)$ by $\ell_{\mathrm{ce}}(f(x), \xi = \tilde{z}_k) := -\log(f_k(x))$.

\emph{Additional Notation.} Due to the compactness of $S$, the cost function $c(\cdot, \cdot)$ is bounded and we define $\bar{c}:= \sup_{w\in S, \xi \in \Xi} |c(w, \xi)|$. Because of the compactness of $S$, we can define diameters of $S$. We let $\mathrm{diam}_j(S) := \sup_{u,v\in S} |u_j-v_j|$ to denote the coordinate-wise diameter of the feasible region $S$. We further let $\mathrm{diam}(S):= \sum_{j=1}^d \mathrm{diam}_j(S)$ denote the summation of the coordinate-wise diameter of all coordinates. 
Given a norm $\|\cdot\|$ defined on $\bbR^d$, the distance from a point $w \in \bbR^d$ to a set $W \subseteq \bbR^d$ is denoted by $\mathrm{dist}(w, W):= \inf_{u\in W} \|w-u\|$.
For a convex function $h(\cdot): S \to \bbR$, we let $\partial h(w)$ denote the set of subgradients of $h(\cdot)$ at $w$.
Let $\circ$ denote the composition of functions. For example, with $f: \calX \rightarrow \Delta_K$ and $w: \Delta_K \rightarrow S$, then $w \circ f$ is the function from $\calX$ to $S$ with $(w \circ f)(x) := w(f(x))$ for all $x\in \calX$. This function composition notation also extends naturally to function classes. For example, for a class $\calH$ of functions $f : \calX \to \Delta_K$, we let $w \circ \calH$ denote the class of functions $\{w \circ f : f \in \calH \}$. For $f, g \in \calH$, recall the sup-norm is defined as as $\|f-g\|_\infty := \sup_{x\in \calX} |f(x) - g(x)|$.  We denote the set of non-negative integers a $\mathbb{N}_0$ and let $\mathbb{N}_0^k$ denote the set of all $k$-dimensional vectors with each component is a  non-negative integer. $\mathds{1}$ denotes the $K$-dimensional vector with all coordinates taking the value of one. 
We let $\mathrm{TV}(\calP, \calQ)$ denote the total variation between two probability measures $\calP$ and $\calQ$ supported on the $K-1$-dimensional simplex $\Delta_K$. $\mathrm{TV}(\calP, \calQ):= \sum_{A\in \calB} |\calP(A) - \calQ(A)|$ where $\calB$ denote the class of Borel sets in $\Delta_K$. In Section \ref{sec: approx-mapping}, we will use an equivalent expression of $\mathrm{TV}(\calP, \calQ) = \frac{1}{2} \sup_{f: \Delta \rightarrow [-1, 1]} (\int_{\Delta_K} f(p) d\calP(p) - \int_{\Delta_K} f(p) d\calQ (p)$).

\subsection{Motivating Examples}
In this section, we present a few motivating examples for the ICEO framework, some of which will be revisited in our numerical experiments in Section \ref{sec:experiments}.

\begin{example}[Multi-item Newsvendor]\label{example:newsvendor} 
The multi-item Newsvendor problem aims to find the optimal replenishment quantities for $d$ different products. We let $\xi := (\xi_1, \dots, \xi_d)$ denote the random demand of $d$ products and let $w \in \bbR^d$ denote the associated order quantities. The demand values $\xi$ might be related to contextual information such as promotions, holiday seasons, brand information, etc.
The objective of this problem is the total inventory cost including the holding costs $h_l$ and stockout costs $b_l$, which characterize the over-stock and under-stock, respectively. 
The objective cost can be formulated as 
\begin{equation}
\label{eq: newsvendor-obj}
    c(w, \xi) := \sum_{l=1}^d h_l(w_l-\xi_l)^+ + b_l(\xi_l - w_l)^+,
\end{equation}
where the function $(\cdot)^+$ is defined as $\max\{\cdot, 0\}$.
Moreover, we consider a budget capacity constraint $C > 0$ on the total order quantities and formulate the feasible set as
\begin{equation*}
    S := \{w : \sum_{l=1}^d w_l \leq C, w\geq 0\}.
\end{equation*}
\label{exp: nv}
\end{example}
\vspace{-5mm}
\begin{example}[Risk-Averse Portfolio Optimization] We consider the problem of finding an optimal risk-averse portfolio of $d$ assets. We denote the random vector of asset returns by $\xi\in \mathbb{R}^d$, which may be associated with contextual information such as economic indicators, news headlines, etc. The decision maker aims to find the best allocation of assets $w \in \mathbb{R}^d$ that optimizes a weighted combination of the expected return and variance of the portfolio. By introducing an auxiliary variable $w_0 \in \bbR$, we formulate the objective as 
\begin{equation}
\label{eq: portfolio-obj}
    c(w, w_0, \xi) := \alpha \left(\sum_{l = 1}^d w_l \xi_l - w_0\right)^2 - \sum_{l = 1}^d w_l \xi_l,
\end{equation}
where $\alpha > 0$ is a trade-off parameter.
Note that the expectation of the first term in \eqref{eq: portfolio-obj} is $\alpha\mathbb{E}_\xi\left[(\sum_{l = 1}^d w_l \xi_l - w_0)^2\right]$, which represents the variance of the investment return $\text{Var}(\sum_{l = 1}^d w_l \xi_l)$ when $w_0$ is optimally selected as $w_0 = \mathbb{E}_\xi\left[\sum_{l = 1}^d w_l \xi_l\right]$, while the second term is the return of the portfolio. Therefore, $\mathbb{E}_\xi[c(w, w_0, \xi)]$ trades off between minimizing the variance and maximizing the expected return of the portfolio. As is standard in the classical portfolio optimization problems, we constrain the portfolio decision in the simplex $\Delta_d = \{w\in \mathbb{R}^d : \sum_{l=1}^d w_l = 1, w \geq 0$\} and we have
\begin{equation*}
    S := \{ (w, w_0): w\in \Delta_d, w_0\geq 0, 0 \leq w_0 \leq \bar{\Xi}\},
\end{equation*}
where $\bar{\Xi} \geq 0$ is a known upper bound the maximum of the returns $\|\xi\|_\infty$.
\end{example}

\begin{example}[Minimum Convex Cost Flow Problem] 
\label{example: networkflow}
Many applications such as urban traffic systems and area transfers in communication networks can be formulated as a minimum convex cost flow problem (we refer to Chapter 14 of \cite{ahuja1988network} for more details).
In the minimum convex cost flow problem, the decision-maker aims to find the flow that minimizes the associated cost on the edges. The cost is a convex function of flow and depends on a random parameter. Suppose we consider a directed graph with $d$ edges and the random parameter $\xi\in \mathbb{R}^d$. In this problem, we consider the objective function 
\begin{equation*}
c(w, \xi) = \sum_{l=1}^d g_l(w_l, \xi_l) 
\end{equation*}
where $g_l$ is a convex function of $w_l$ and $g_l$ can be different for different coordinates.
Similar to the standard network flow problem, we let the matrix $A$ denote the node-arc incidence matrix of the graph and restrict the flow on each edge in the region $[l,u]$. Therefore, we have the feasible region
\begin{equation*}
    S = \{w \in \mathbb{R}^d: Aw = 0, w\in [l,u]^d \}.
\end{equation*}
\end{example}

\section{Performance Guarantees}
\label{sec: general-guarantee}
In this section, we demonstrate asymptotic consistency and finite-sample performance guarantees of the ICEO approach. Let us first introduce some additional notation.
We state our results in terms of arbitrary policy mappings $\pi : \mathcal{X} \rightarrow S$, which represent any mapping from the feature space $\mathcal{X}$ to the set of feasible decisions $S$. Our main interest herein is the class of policies that combine the optimal solution mapping and hypothesis $f$, i.e., $\Pi = w_\rho \circ \calH$. This class of policies could include both hypotheses learned by the ICEO approach as well as by ETO approaches. In the remaining part of this work, we let $f^\ast: \calX \to \Delta_K$ denote the function that maps from $x$ to the true conditional distribution $p^\ast(x)$. We refer to $f^\ast$ as the true hypothesis. 
To quantify our performance guarantees, we define the following risk functions for any policy $\pi$: 
\begin{enumerate}
     \item $\hat{R}_n(\pi)$:  The empirical risk with respect to a given sample $\{(x_i, \xi_i)\}_{i=1}^n$, i.e.,
     \begin{equation*}
     \hat{R}_n(\pi) :=  \frac{1}{n}\sum_{i=1}^n c(\pi(x_i), \xi_i). 
     \end{equation*}
     \item $R(\pi)$: The expected risk with respect to the underlying joint distribution $\calD$ of $x$ and $\xi$, i.e.,
     \begin{equation*}
        \label{eq: l-defn}
        R(\pi):= \mathbb{E}_{x, \xi}\left[c(\pi(x), \xi ) \right]= \mathbb{E}_x \left[\sum_{k=1}^K p^\ast_k(x)c_k(\pi(x)) \right],
     \end{equation*}
     where $p^\ast_k(x) = \bbP_\xi(\xi = \tilde{z}_k | x)$ for all $k = 1, \dots, K$.
\end{enumerate}

 Note that $\hat{R}_n(\cdot)$ is the objective function of (\ref{program:origin}). 
 We use the notation $\calD_x$ to refer to the marginal distribution of the features $x$.
 We further define the optimal risk values for the class of policies $w_\rho \circ \calH$ and $w \circ \calH$ that we consider herein. Recall that, for all $k$, $c_k(\cdot)$ is convex but in general not strongly convex. Therefore, there may exist multiple optimal solutions, as denoted by $W(p)  \argmin_{w\in S} \sum_{k=1}^K p_kc_k(w)$, for any input probability vector $p\in \Delta_K$. In this case, we assume that the oracle $w(\cdot) : \Delta_K \to S$ as a function that arbitrarily outputs a value from the optimal solution set $W(\cdot)$, i.e., $w(p) \in W(p)$ for all $p \in \Delta_K$. 
\begin{enumerate}
    \item $J^*$: the optimal expected unregularized risk, i.e.,
    \begin{equation*}
    \label{eq: j-star}
        J^* := \min_{f\in \mathcal{H}} \mathbb{E}_x\left[\sum_{k=1}^K p_k^*(x)c_k(w(f(x)))\right] = \min_{f\in \mathcal{H}} R(w \circ f).
    \end{equation*}
    In Lemma \ref{lemma:f-star-minimizer} in the Appendix, we demonstrate that the value of the unregularized optimal risk $J^\ast$ does not depend on the particular choice of $w(\cdot)$ and that $f^\ast$ is always the minimizer that achieves $J^\ast$.
    \item $J^*_\rho$: the optimal expected regularized risk for any given regularization parameter $\rho > 0$, i.e.,
    \begin{equation*}
     J^*_\rho := \min_{f\in \mathcal{H}} \mathbb{E}_x \left[\sum_{k=1}^K p^*_k(x)c_k(w_\rho(f(x)))\right] = \min_{f\in \mathcal{H}} R(w_\rho\circ f).
    \end{equation*}
     and we let $f^\ast_\rho$ denote an optimal hypothesis. 
    \item $\hat{J}^n_\rho$: the optimal empirical risk with any given sample $S_n$ and a given regularization parameter $\rho >0$, i.e.,
    \begin{equation*}
    \hat{J}^n_\rho := \min_{f\in \mathcal{H}} \frac{1}{n}\sum_{i=1}^n c(w_{\rho}(f(x_i)), \xi_i)) =\min_{f\in \mathcal{H}} \hat{R}_n(w_\rho\circ f),
    \end{equation*}
    and we let $\hat{f}^n_\rho$ denote an optimal hypothesis. 
\end{enumerate}

\subsection{Asymptotic Consistency}\label{sec:consistency}
We first demonstrate the asymptotic consistency of the ICEO approach. The consistency of ICEO is three-fold: the consistency of the ICEO risk, the consistency of the ICEO decisions, and the consistency of the ICEO hypothesis.
The asymptotic consistency results build upon the convergence of the regularized oracle to the original optimization oracle, as stated in Proposition \ref{prop: oracle-consistency}.
\begin{proposition}[Convergence of Regularized Oracle]
\label{prop: oracle-consistency}
Suppose $w_\rho(\cdot)$ is the regularized optimal solution mapping for any $\rho >0$, as defined in \eqref{program:origin-lower-level_regularized}. For any positive sequence $\{\rho_n\}$ that satisfies $\lim_{n\rightarrow \infty}\rho_n =0$, and for any $p\in \Delta^K$, we have $\mathrm{dist}(w_{\rho_n}(p), W(p))\rightarrow 0$ as $n \to \infty$.
\end{proposition}

To establish asymptotic consistency, we require the following assumptions concerning the hypothesis class $\mathcal{H}$ and the regularized oracle $w_\rho(\cdot)$.
A standard sufficient (but not necessary) condition to ensure uniform convergence and thus asymptotic consistency in learning theory is for the hypothesis class to have a finite bracketing number (we refer to \cite{vaart2023empirical} for more details). We build upon this natural sufficient condition by introducing the \emph{multivariate bracketing number} of the hypothesis class $\calH$. In other words, we present a modest extension of the standard definition of the bracketing number for real-valued functions to the multivariate case. Herein, we first define a multivariate bracket of $\mathcal{H}$ relative to any norm $\|\cdot\|$ on $\mathcal{H}$.
\begin{definition}[Multivariate $\epsilon$-bracket]
Given two functions $l : \mathcal{X} \to \mathbb{R}^K$ and $u : \mathcal{X} \to \mathbb{R}^K$, the bracket $[l,u]$ is the set of all functions $f\in \mathcal{H}$ with $l_k(x) \leq f_k(x) \leq u_k(x)$ for each coordinate $k = 1, \dots, K$ and for all $x\in \mathcal{X}$. An $\epsilon$-bracket is a bracket $[l,u]$ with $\|l-u\|< \epsilon$. 
\end{definition}
Then the multivariate bracketing number can be defined as follows.
\begin{definition}[Multivariate bracketing number]
\label{def: bracketing-n}
The multivariate bracketing number $N_{[]}(\epsilon, \mathcal{H}, \|\cdot\|)$ is the minimum number $N$ of $\epsilon$-brackets $[l_1,u_1], \ldots, [l_N, u_N]$ that cover $\mathcal{H}$, i.e., with the property that for all $f \in \mathcal{H}$ there exists $i \in \{1, \ldots, N\}$ such that $f \in [l_i,u_i]$.
\end{definition}

Establishing the asymptotic consistency of our proposed ICEO method requires that $\calH$ be compact and that its multivariate bracketing number is finite. Note that the metric defined on $f$ is the sup-norm, $\|f-g\|_\infty := \sup_{x\in \calX} |f(x) - g(x)|$. 
\begin{assumption}
For the compact hypothesis class $\calH$, we assume the following properties:
\begin{enumerate}[label={\Alph*.},
  ref={\theassumption.\Alph*}]
    \item{(Model specification.)} \label{assump: specification} The hypothesis class $\calH$ includes the true hypothesis $f^*$  i.e., $f^*\in \calH$.
    \item{(Finite bracketing number of $\mathcal{H}$.)} \label{assump: finite-bracket-N} The multivariate bracketing number, $N_{[]}(\epsilon, \mathcal{H}, \|\cdot\|_\infty)$, as defined in Definition \ref{def: bracketing-n}, is finite for any $\epsilon \in (0,1)$.
    \item{(Unique optimal hypotheses.)} \label{assump: unique}There does not exists a hypothesis $f\neq f^*$ in $\calH$ such that $W(f(x)) \cap W(f^*(x)) \neq \emptyset$, $\mathcal{D}_x$-almost surely for all $x\in \mathcal{X}$. For $\rho >0$ and for all $f'\neq f^*_\rho$ in $\calH$, there exists $\epsilon >0$ with $\mathbb{P}_x(\|w_\rho(f'(x)) - w_\rho(f^*_\rho(x))\| \geq \epsilon) >0$.
\end{enumerate}
\end{assumption}

Regularity conditions of the oracle are also needed to establish the asymptotic consistency of the ICEO method. To be more specific, we assume that the regularized oracle is uniformly Lipschitz, as presented in Assumptions \ref{assump: oracle-uni-lip-rho}-\ref{assump: oracle-uni-lip}. The third assumption \ref{assump: unique-decision} is only needed for the consistency of decisions and hypotheses. 

\begin{assumption}[Uniform Properties of the Regularized Oracle]
For the regularized oracle, $w_\rho(\cdot)$ with $\rho \in (0, \rho_0]$ for some $\rho_0 > 0$, we assume the following properties:
\begin{enumerate}[label={\Alph*.},
  ref={\theassumption.\Alph*}]
     \item{(Lipschitz in $\rho$.)} \label{assump: oracle-uni-lip-rho} For all $p\in \Delta_K$, $\|w_{\rho_1}(p) - w_{\rho_2}(p)\| \leq L_\rho |\rho_1-\rho_2|$ for all $\rho_1, \rho_2 \in (0, \rho_0]$.
     \item {(Lipschitz in $p$.)}\label{assump: oracle-uni-lip} For all $\rho\in (0, \rho_0]$, $\|w_{\rho}(p) - w_{\rho}(p)\| \leq L_w \|p_1 - p_2\|$.
    \item{(Unique optimal decisions.)}  \label{assump: unique-decision} For any $\epsilon > 0$, there exists $\delta > 0$, so that for all $\rho \in (0, \rho_0]$ and $f \in \mathcal{H}$ it holds that
    \begin{equation*}
        \bbP_{x}(\|w_\rho(f(x)) - w_\rho(f^*_\rho(x))\| \geq \epsilon) > 0 ~\Rightarrow~ R(w_\rho \circ f) - R(w_\rho \circ f^*_\rho) \geq \delta.
    \end{equation*}
\end{enumerate}
\end{assumption}
To further justify these assumptions, in Appendix \ref{sec: consistency-proof-appendix}, we study a reasonable sufficient condition to guarantee these assumptions of uniform Lipschitzness. Namely, we demonstrate that when the underlying contextual stochastic optimization problem \eqref{eq: sto-opt-2} satisfies a linear growth condition away from the optimal solution set, then the regularized oracle satisfies an ``automatic crossover'' property. The automatic crossover property says that the regularized oracle automatically outputs a point from the unregularized optimal solution set $W(p)$, whenever $\rho$ is smaller than a certain ``phase-transition'' threshold. Whenever the automatic crossover property holds, the uniform Lipschitz assumptions hold and thus imply uniform equicontinuity of the regularized oracle, which allows us to generalize the pointwise convergence of the oracle (Proposition \ref{prop: oracle-consistency}) to uniform convergence.

To guarantee the consistency of the ICEO method, we consider a sequence of regularization parameters $\rho_n$, depending on the sample size $n$, such that $\rho_n$ converges to zero as $n$ grows to infinity.
Theorem \ref{thm: consistency} below demonstrates the three levels of consistency.
\begin{theorem}[Asymptotic Consistency of ICEO]
\label{thm: consistency}
Suppose that the training data $(x_i, \xi_i)$ is an i.i.d. sequence from the distribution $\mathcal{D}$ and that the sequence of regularization parameters $\rho_n \in (0, \rho_0]$ satisfies $\lim_{n\rightarrow \infty}\rho_n =0$. Then, under Assumptions \ref{assump: specification}~-~\ref{assump: finite-bracket-N} and Assumptions \ref{assump: oracle-uni-lip-rho}~-~\ref{assump: oracle-uni-lip}, we have the following:
\begin{itemize}
    \item[(i)] The optimal empirical regularized risk converges to the optimal expected risk, i.e., $\hat{J}^n_{\rho_n} \rightarrow J^*$ with probability 1.
    \item[(ii)] Additionally, with Assumptions \ref{assump: unique} and \ref{assump: unique-decision}, $\mathcal{D}_x$-almost surely for all $x \in \mathcal{X}$, the sequence of ICEO decisions $w_{\rho_n}(\hat{f}^n_{\rho_n}( x))$ converges to the true set of optimal decisions $W(f^*(x))$, i.e., $\mathrm{dist}(w_{\rho_n}(\hat{f}^n_{\rho_n}( x)), W(f^*(x))) \to 0$ with probability 1.
    \item[(iii)] Additionally, with Assumptions \ref{assump: unique} and \ref{assump: unique-decision}, the sequence of ICEO hypotheses converges to the true hypothesis, i.e, $\hat{f}^n_{\rho_n} \rightarrow f^*$ with probability 1. 
\end{itemize}
\end{theorem}
The proof of Theorem \ref{thm: consistency} can be found in Appendix \ref{sec: consistency-proof-appendix}. The idea of the proof, particularly for part {\em (i)}, is to separately establish the convergence of $J^*_\rho$ to $J^\ast$ and the uniform convergence over $\rho$ of $\hat{J}_\rho^n$ to $J^\ast_\rho$. 
The first technical challenge involved is demonstrating the convergence properties of the regularized optimal solution oracle $w_\rho(\cdot)$ to the optimal solution set $W(\cdot)$, even beyond the pointwise convergence provided in Proposition \ref{prop: oracle-consistency}. These properties are needed for showing the convergence of both $J^*_\rho$ to $J^\ast$ and $\hat{J}_\rho^n$ to $J^\ast_\rho$. 
Secondly, to establish the convergence of $\hat{J}_\rho^n$ to $J^\ast_\rho$ uniformly over $\rho$ as we collect more data, the convergence of the oracle is not enough. Indeed, we demonstrate finiteness of the bracketing number of the composition class $\mathcal{F}:= \{g : g = c \circ w_\rho \circ f \text{ for some } \rho \in (0, \rho_0] \text{ and } f\in \mathcal{H}\}$ (Lemma \ref{lemma: finite-cover-num}).

We would like to clarify the relationship between the asymptotic consistency stated in Theorem \ref{thm: consistency} and the asymptotic optimality defined in \cite{bertsimas2020predictive}. In \cite{bertsimas2020predictive}, the authors define asymptotic optimality in terms of the decisions reaching the best objective function performance possible. Because of the continuity of the cost function $c$, the convergence of ICEO decisions, as stated in (ii) of Theorem \ref{thm: consistency}, implies the asymptotic optimality property of \cite{bertsimas2020predictive}.

\subsection{Finite Sample Performance Guarantees}\label{sec:generalization}
We now provide finite sample performance guarantees of the ICEO solution $\hat{f}^n_{\rho_n}$ in the form of generalization bounds based on Rademacher complexities. 
In particular, our overall strategy is as follows: {\em (i)} we demonstrate that, due to the presence of the strongly convex decision regularization function $\phi(\cdot)$, the optimal solution mapping $w_\rho(\cdot)$ is Lipschitz, {\em (ii)} we use the result of \cite{maurer2016vector} to bound the Rademacher complexity with respect to the cost function of the ICEO framework by the multivariate Rademacher complexity of the underlying hypothesis class $\mathcal{H}$.
In addition, we slightly abuse the notation and let $c(\cdot) : S \to \bbR^K$ denote a vector-valued mapping, where each component $c_k(w)$ denotes the cost $c(w, \xi=\tilde{z}_k)$ for all scenarios $k = 1, \dots, K$, as defined earlier in Section \ref{sec: framework}.

Before we investigate the Rademacher complexities, we first demonstrate the Lipschitz property of the regularized optimal solution mapping $w_\rho(\cdot)$ for any positive parameter $\rho$, based on the following assumption regarding the Lipschitz property of the cost function $c(w)$ and the strong convexity constant of the decision regularization function $\phi(\cdot)$.
\begin{assumption}
\label{assump:c-lip-phi-constant}
The cost function $c(\cdot)$ and the decision regularization function $\phi(\cdot)$ satisfy the following conditions:
\begin{enumerate}[label={\Alph*.},
  ref={\theassumption.\Alph*}]
    \item  $c(\cdot)$ is $L_c$-Lipschitz with respect to the decision $w \in S$, i.e., it holds that $\|c(w_1) - c(w_2)\|_2 \leq L_{c}\|w_1-w_2\|$ for all $w_1, w_2 \in S$. \label{assump:c-lip}
    \item The decision regularization function $\phi(\cdot)$ is a 1-strongly convex function on the compact set $S$. \label{assump:phi-constant}
\end{enumerate}
\end{assumption}
Note that we use the $\ell_2$ norm as the norm on the space of outputs of the cost functions $c(\cdot)$, while the norm on the space of decisions $w$ remains the generic norm $\|\cdot\|$. The reason for focusing on the $\ell_2$ norm is that we can apply the elegant vector contraction inequality of \cite{maurer2016vector} when analyzing the Rademacher complexity.
It is also worth mentioning that the Lipschitz condition in Assumption \ref{assump:c-lip} implies that the cost functions $c_k(\cdot)$ are uniformly $L_c$-Lipschitz, i.e., $\|c(w_1) - c(w_2)\|_\infty \leq L_{c}\|w_1-w_2\|$. 
\begin{proposition}[Lipschitz Properties of $w_\rho(\cdot)$ and $c(\cdot)$]
Suppose Assumption \ref{assump:c-lip-phi-constant} holds and note that $\bbR_+^{K}:=\{ p\in \bbR^K: p_k \geq 0,  \forall k=1, \dots, K\}$. Then, for any $\rho > 0$, the optimal solution mapping $w_\rho(\cdot)$ is $(\frac{L_{c}}{\rho})$-Lipschitz:
\begin{equation}
\label{eq: w-lip}
\|w_\rho(p) - w_\rho(p^\prime)\| \leq \frac{L_{c}}{\rho} \|p-p^\prime\|_2,\qquad \forall p, p^\prime \in \bbR_+^{K} %
\end{equation}
Furthermore, $c(w_\rho(\cdot))$ is $(\frac{L_{c}^2}{\rho})$-Lipschitz:
\begin{equation}
\label{eq: c-eq-lip}
\|c(w_\rho(p)) - c(w_\rho(p^\prime))\|_{\infty} \leq \|c(w_\rho(p)) - c(w_\rho(p^\prime))\|_2 \leq \frac{L_{c}^2}{\rho} \|p-p^\prime\|_2,\qquad \forall p, p^\prime \in \bbR_+^{K}\end{equation}\label{prop:w-lip}
\end{proposition}
\vspace{-7mm}
The proof of this Proposition follows standard arguments of Nesterov's smoothing technique (\cite{nesterov2003introductory}), and a related result with a similar proof style appears in \cite{gupta2021data}. Detailed proof can be found in Appendix \ref{sec: appendix-general-guarantee}.

To establish the generalization bound for the ICEO risk, we rely on both regular single-variate and multi-variate Rademacher complexity. In the ICEO setting, given a class of policies $\Pi$, where $\pi : \calX \to S$ for all $\pi \in \Pi$, we can apply generalization bounds that directly use the Rademacher complexity of the function class $c\circ \Pi$. Given a sample $\{(x_i, \xi_i)\}_{i=1}^n$ the \emph{empirical Rademacher complexity} $\hat{\mathfrak{R}}_n(c\circ \Pi)$ of the function class $c\circ \Pi$ is defined by 
\begin{equation*}
    \hat{\mathfrak{R}}_n(c\circ \Pi) ~:=~ \mathbb{E}_\sigma\left[ \frac{2}{n}\sup_{g\in c\circ \Pi} \sum_{i=1}^n \sigma_{i} g(x_i, \xi_i)\right] =  \mathbb{E}_\sigma\left[ \frac{2}{n}\sup_{\pi \in  \Pi} \sum_{i=1}^n \sigma_{i}  c(\pi(x_i), \xi_i)\right],
\end{equation*}
where $\sigma_{i}$ are independent random variables drawn from the Rademacher distribution, i.e. $\Pr(\sigma_{i} = +1)=\Pr(\sigma_{i} = -1) = \frac{1}{2}$ for all $i = 1,2,\dots, n$. The \emph{expected Rademacher complexity} $\mathfrak{R}_n(c\circ \Pi)$ is then defined as the expectation of $\hat{\mathfrak{R}}_n(c\circ \Pi)$ with respect to the i.i.d. sample $\{(x_i, \xi_i)\}_{i=1}^n$ drawn from the distribution $\calD$:
\begin{equation*}
  \mathfrak{R}_n(c\circ \Pi) = \bbE_{(x_i, \xi_i)\sim\calD}[\hat{\mathfrak{R}}_n(c\circ \Pi)].
\end{equation*}
Next, we introduce the multivariate Rademacher complexity as a generalization of the regular Rademacher complexity to a class of vector-valued functions. In the ICEO context, we focus on the hypothesis class $\mathcal{H}$ which takes values in $\Delta_K$. Following \cite{bertsimas2020predictive}, \cite{maurer2016vector} and \cite{balghiti2019generalization}, the \emph{empirical multivariate Rademacher complexity} $\hat{\mathfrak{R}}_n(\mathcal{H})$ is defined in our context as
\begin{equation*}
    \hat{\mathfrak{R}}_n(\mathcal{H}) = \mathbb{E}_\sigma\left[ \frac{2}{n}\sup_{f\in \mathcal{H}} \sum_{i=1}^n 
    \sum_{k = 1}^K \sigma_{ik} f_k(x_i)\right],
\end{equation*}
where $\sigma_{ik}$ are also independent random variables drawn from the Rademacher distribution for all $i = 1,2,\dots, n$ and $k = 1, \ldots, K$. Correspondingly, the \emph{expected multivariate Rademacher complexity} $\mathfrak{R}_n(\mathcal{H})$ is then defined as 
\begin{equation*}
    \mathfrak{R}_n(\mathcal{H}) =\bbE_{x_i\sim\calD_x}[\hat{\mathfrak{R}}_n(\mathcal{H})],
\end{equation*}
In the remainder of this section, we provide generalization bounds with respect to the expected single-variate and multi-variate Rademacher complexities. We note that similar results can be achieved with respect to the empirical versions of the Rademacher complexities. Our focus on the expected versions is justified since, for many hypothesis classes $\mathcal{H}$, we can bound $\mathfrak{R}_n(\mathcal{H})$ by a term that converges to $0$ as the sample size $n$ grows. 
For example, \cite{balghiti2019generalization} establish upper bounds of $\mathfrak{R}_n(\mathcal{H})$ for regularized linear hypothesis classes with the rate of $\mathcal{O}(\frac{1}{\sqrt{n}})$, where the $\mathcal{O}(\cdot)$ notation hides dimension dependent constants that depend on the type of regularization used.

Given a sample $\{(x_i, \xi_i)\}_{i=1}^n$, we aim to provide a high-probability bound on the out-of-sample risk $R(w_{\rho_n} \circ f)$, given the in-sample risks $\hat{R}_n(w_{\rho_n} \circ f)$ and $\hat{R}_n(w_{\rho_n} \circ f; \rho_n)$, that holds uniformly for any hypothesis $f \in \mathcal{H}$.
As such, our generalization bound is constructed based on the classic generalization bound with Rademacher complexity due to \cite{bartlett2002rademacher}, which we restate below as specialized to the ICEO setting. Recall that $\bar{c}:= \sup_{w\in S, \xi \in \Xi} c(w, \xi)$.
\begin{theorem}[\cite{bartlett2002rademacher}]
\label{thm: bartlett}
Let $\Pi$ be a family of functions mapping from $\calX$ to $S$ with bounded Rademacher complexity $\mathfrak{R}_n(c\circ \Pi)$. Then, for any $\delta \in (0,1]$, with probability at least $1 - \delta$ over i.i.d. data $\{(x_i, \xi_i)\}_{i=1}^n$ drawn from the distribution $\calD$, the following inequality holds for all $\pi \in \Pi$:
\begin{equation*}
    R(\pi) ~\leq~  \hat{R}_n(\pi) + \mathfrak{R}_n(c\circ \Pi) + \bar{c}\sqrt{\frac{\log(\frac{1}{\delta})}{2n}}.
\end{equation*}
\end{theorem}
The next step of our analysis is to apply the vector contraction inequality of \cite{maurer2016vector} to derive a generalization bound that depends directly on the multi-variate Rademacher complexity of the hypothesis class $\mathcal{H}$.

\begin{theorem}[Generalization of ICEO]
Suppose Assumption \ref{assump:c-lip-phi-constant} holds and that the hypothesis class $\calH$ has bounded multi-variate Rademacher complexity $\mathfrak{R}_n(\mathcal{H})$. Then, for any $\delta \in (0,1]$ and $\rho_n > 0$, with probability at least $1 - \delta$ over i.i.d. data $\{(x_i, \xi_i)\}_{i=1}^n$ drawn from the distribution $\calD$, the following inequalities hold for all $f \in \calH$:
\begin{align}
\label{eq: finite-sample-bd-no-phi}
R(w_{\rho_n}\circ f) &~\leq~ \hat{R}_n(w_{\rho_n} \circ f ) + \frac{\sqrt{2}L_c^2}{\rho_n}\mathfrak{R}_n(\mathcal{H}) + \bar{c}\sqrt{\frac{\log(\frac{1}{\delta})}{2n}}.
\end{align}
\label{thm:finite-sample-bd}
\end{theorem}
Note that the right-hand side of the inequality in Theorem \ref{thm:finite-sample-bd} involves the non-regularized empirical risk, which may be evaluated for any $f \in \calH$. This is also the objective function of \eqref{program:origin}.
As mentioned previously, one can often establish upper bounds on $\mathfrak{R}_n(\mathcal{H})$ that converge to zero, for example at the rate $\mathcal{O}(\frac{1}{\sqrt{n}})$. Therefore, Theorem \ref{thm:finite-sample-bd} suggests that we should set the sequence of regularization parameters $\rho_n$ so that $\mathfrak{R}_n(\mathcal{H})/\rho_n$ converges to zero as well, in which case the remainder terms on the right-hand side of (\ref{eq: finite-sample-bd-no-phi}) converge to zero. We now state the proof of Theorem \ref{thm:finite-sample-bd}. Proof of of Theorem \ref{thm:finite-sample-bd} utilizes Proposition \ref{prop:w-lip} and can be found in Appendix \ref{sec: appendix-general-guarantee}

Leveraging Theorem 3, we obtain the following corollary that provides an upper bound for the suboptimality gap, in terms of out-of-sample risk evaluation, for the ICEO minimizer $\hat{f}^n_{\rho_n}$ relative to the out-of-sample risk of the optimal in-class hypothesis and the true hypothesis.

\begin{corollary}
\label{cor: optimal-out-of-sample-bound}
Suppose Assumption \ref{assump:c-lip-phi-constant} holds and that the hypothesis class $\calH$ has bounded multi-variate Rademacher complexity $\mathfrak{R}_n(\mathcal{H})$. Let $f^*_{\calH} \in \arg\min_{f \in \calH} R(w \circ f)$ denote an optimal in-class hypothesis.
Then, for any $\delta \in (0,1]$ and $\rho_n > 0$, with probability at least $1 - \delta$ over i.i.d. data $\{(x_i, \xi_i)\}_{i=1}^n$ drawn from the distribution $\calD$, the following inequality holds:
\begin{equation*}
   R(w_{\rho_n} \circ \hat{f}^n_{\rho_n}) - R(w_{\rho_n} \circ f^*_{\calH})
   ~\leq~ \frac{\sqrt{2}L_c^2}{\rho_n}\mathfrak{R}_n(\mathcal{H}) +\frac{3\bar{c}}{2} \sqrt{\frac{2\log(\frac{2}{\delta}) }{n}}.
\end{equation*}
If, in addition, we have $f^* \in \mathcal{H}$, then we have
\begin{equation*}
   R(w_{\rho_n} \circ \hat{f}^n_{\rho_n}) - R(w_{\rho_n} \circ f^\ast)
   ~\leq~ \frac{\sqrt{2}L_c^2}{\rho_n}\mathfrak{R}_n(\mathcal{H}) +\frac{3\bar{c}}{2} \sqrt{\frac{2\log(\frac{2}{\delta}) }{n}}.
\end{equation*}
\end{corollary}

\begin{remark}
The generalization bound provided in Theorem \ref{thm:finite-sample-bd} involves the Rademacher complexity $\mathfrak{R}_n(\mathcal{H})$ of the hypothesis class $\calH$, which returns outputs in $\Delta_K$. It is noteworthy that, if $\calH$ involves a softmax operator, i.e., $\calH = \mathrm{soft} \circ \tilde{\calH}$ for some hypothesis class $\tilde{\calH}$ outputting in $\bbR^K$, then it is possible to obtain an identical generalization bound involving the Rademacher complexity $\mathfrak{R}_n(\tilde{\calH})$ of the hypothesis class $\tilde{\calH}$. Indeed, since the softmax function $\mathrm{soft}$ is Lipschitz (\cite{gao2017properties}) with constant 1, the Lipschitz bound \eqref{eq: c-eq-lip} can be extended to a bound involving the outputs of $\tilde{\calH}$ with the same $(L_c^2/\rho)$ constant. Alternatively and equivalently, the vector contraction inequality of \cite{maurer2016vector} can be extended to a Lipschitz mapping (such as the softmax function) that has both multivariate input and output. This extended vector contraction inequality can be obtained by a straightforward extension of Lemma 7 and Theorem 3 of \cite{maurer2016vector} to the multivariate case.
\end{remark}

\section{Computational Methods}
\label{sec: computational}
In this section, we discuss the computational difficulties of solving the ICEO formulation \eqref{program:origin} and present multiple approaches to address them.

\paragraph{\textbf{Non-convexity.}} First, we point out that the ICEO formulation, \eqref{program:origin}, is not a convex optimization problem even in a very simple case where both the objective and constraints of the nominal optimization problem are linear and the decision regularization is quadratic, as stated in Example \ref{exp: non-linear} in Appendix \ref{sec: appendix-computation}.
This example shows that, even in this simplest case, (\ref{program: lp-example}) is not a convex optimization problem. Besides non-convexity, a more serious issue from a practical standpoint is the potential of non-differentiability of optimal solution mapping.

\paragraph{\textbf{Non-differentiability.} }
To solve the non-convex ICEO problem \eqref{program:origin}, a default approach in machine learning is to use a gradient-based algorithm such as the basic stochastic gradient descent method. 
Indeed, in practice, gradient-based algorithms are often able to deliver high-quality solutions for machine learning problems, especially in high dimensions. 
Unfortunately, applying these basic gradient-based methods to solve the ICEO formulation poses an additional major difficulty due to the non-differentiability of the optimal solution mapping $w_\rho(\cdot)$.
Although $w_\rho(\cdot)$ is a continuous function, as guaranteed for example by Proposition \ref{prop:w-lip}, it is generally not differentiable. 
The non-differentiability leads to major difficulties in applying  gradient-based methods while solving \ref{program:origin}. 
As reviewed in Section \ref{sec: literature}, existing studies that focused on directly learning the optimal solution mapping $w(f^*(x))$ also encounter the same issue of non-differentiability. \cite{wilder2019end} does not discuss much about this. \cite{donti2017task} use the output of an automatic gradient function calculated by back-propagation of a neural network. \cite{agrawal2019differentiable} approximate the gradient by solving a group of linear equations based on KKT conditions. However, all existing methods fail to demonstrate theoretical reliability or performance guarantees in approximating the gradient.

The non-differentiability of the optimal solution mapping mainly arises from the constraints and the points of discontinuity occur where there is a ``jump" in the optimal solution, e.g., in the polyhedral case as in Example \ref{exp: non-linear}. Therefore, a non-differentiable optimal solution map may also have regions where it is constant (or close to constant), resulting in the gradient of the ICEO objective being equal to zero. 
We demonstrate this poor behavior in Figure \ref{fig:oracle}, where we plot the second coordinate of the optimal solution mapping $w_\rho(\cdot)$ with respect to the first two coordinates of the input probability vector, for the multi-product newsvendor problem in Example \ref{exp: nv}, demonstrating the piece-wise constant shape. Such piece-wise constant shapes will greatly impede the performance of gradient-based methods, even if the gradient is easily calculated. This is because the gradient of the optimal solution mapping is zero in flat regions creating poor local minima that are very difficult to escape. 

\begin{figure}[ht]
\begin{subfigure}{.4\textwidth}
  \centering
  \includegraphics[width=.75\linewidth]{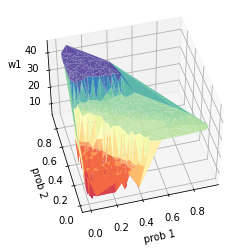}  
  \caption{3-D plot of the optimal oracle. It is clear that the landscape of the optimal solution mapping has cliffs and platforms. }
  \label{fig:oracle}
\end{subfigure}
\hspace{7mm}
\begin{subfigure}{.4\textwidth}
  \centering
  \includegraphics[width=.75\linewidth]{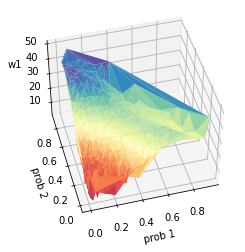}
  \caption{3-D plot of the approximated oracle constructed using polynomial functions. The piece-wise constant shape is smoothed out. }
  \label{fig:approx-oracle}
\end{subfigure}
\caption{The landscape of the optimal and the approximated oracle.}
\label{fig:fig}
\end{figure}
To address the issue of non-differentiability and its consequences leading to poor local optima and slow convergence, we develop a framework for approximating the mapping $w_\rho(\cdot)$ with a differentiable function $\tilde{w}_\rho(\cdot)$, which allows us to smooth out the optimal solution mapping and enhance convergence to a good local optimum.
Figure \ref{fig:approx-oracle} is an example of smoothing out the piece-wise constant shape by constructing an approximate oracle using polynomial kernel regression. As noted before, gradient-based methods are often highly effective at delivering high quality solutions to non-convex machine learning problems in practice. Thus, in a practical sense, the non-differentiability of the optimal solution mapping is a much more serious concern than the non-convexity. Our general strategy of approximating the optimal solution mapping with a differentiable function, for which we expand upon and give examples in Section \ref{sec: approx-mapping}, greatly increases the practical viability of the ICEO approach.

\subsection{Approximating Optimal Solution Mappings}
\label{sec: approx-mapping}
In this section, we propose a general methodology for approximating the potentially non-differentiable optimal solution mapping. In Appendix \ref{sec: approximation-polynomials}, In we provide two examples of using polynomial functions to approximate the optimal solution mapping.

As stated in the previous section, the major computational difficulty in solving the ICEO training problem \eqref{program:origin} in practice is the non-differentiability of the mapping $w_\rho(\cdot)$. 
To overcome this difficulty, for any given $\rho$, we approximate the function $w_\rho(\cdot)$ with a differentiable function $\tilde{w}_\rho(\cdot):\Delta_K \rightarrow S$. Then instead of \eqref{program:origin}, we solve the following problem: 
\begin{align}
\label{program: approx-regularized}
\tag{Approx-ICEO-$\rho$}
    \min_{f\in \mathcal{H}} \quad & \frac{1}{n} \sum_{i = 1}^n c(w_i, \xi_i) \\
    \mathrm{s.t.} \quad & w_i =  \tilde{w}_{\rho}(f(x_i)) \nonumber
\end{align}
To construct such an approximation $\tilde{w}_\rho(\cdot)$, we rely on the ability to evaluate the optimal solution mapping $w_\rho(p)$ for any given $p \in \Delta_K$, as stated in Section \ref{sec: framework}. We can then generate a sequence of samples $(p_i, w_\rho(p_i))$ and build an approximation function $\tilde{w}_\rho(\cdot)$ using any class of continuous functions with enough representation power, such as polynomial functions or neural networks.

We consider two generic types of approximation schemes for building the mapping $\tilde{w}_\rho(\cdot)$: {\em (i)} uniform approximations, and {\em (ii)} high-probability approximations. Uniform approximation schemes satisfy a uniform error bound, as formalized below in Assumption \ref{assump: uniform-err-bd}, and can be achieved by an interpolation method such as the Bernstein polynomial method as described in Section \ref{sec: bernstein}.
Note that, for each $j = 1, \ldots, K$ and $p \in \Delta_K$, we use the notation $w_{\rho,j}(p)$ and $\tilde{w}_{\rho,j}(p)$ to refer to the $j^{\text{th}}$ coordinates of $w_{\rho}(p)$ and $\tilde{w}_{\rho}(p)$, respectively.

\begin{assumption}[Uniform Error Bound]
\label{assump: uniform-err-bd}
For each $j = 1, \ldots, K$, there exists a constant $\calE^\mathrm{unif}_j \geq 0$ such that the approximate optimal solution mapping $\tilde{w}_\rho(\cdot) : \Delta_K \to S$ satisfies:
\begin{equation*}
|w_{\rho,j}(p) -\tilde{w}_{\rho,j}(p)| ~\leq~ \calE^\mathrm{unif}_j,    \qquad \forall p\in \Delta_{K}.
\end{equation*}
\end{assumption}
The uniform error bound in Assumption \ref{assump: uniform-err-bd} provides guarantees for the approximation error over all probability vectors from the simplex $\Delta_K$. There are two main drawbacks that apply to all known approaches for achieving a uniform error bound. First, achieving a tight uniform error bound requires exact or near-exact computations of the optimal solution mapping $w_\rho(p)$ for all $p \in \Delta_K$. In practice, we may only have an approximate optimal solution mapping available. Second, the sample size required by a method that achieves Assumption \ref{assump: uniform-err-bd}, for example, an interpolations scheme, may be prohibitively large. For these reasons we are motivated to consider a high-probability error bound, which would hold for the more realistic approach of using a regression method, possibly with noise in the output of $w_\rho(\cdot)$, to fit the approximate optimal solution mapping. We consider a generic approach that uses a hypothesis class $\mathcal{G}$ for the approximate optimal solution mappings. Assumption \ref{assump: high-prob-err-bd} below formalizes our high-probability error bound, which holds for a wide range of regression methods including, for example, the polynomial kernel regression method considered in Section \ref{sec: krr}. In Assumption \ref{assump: high-prob-err-bd}, we work with a \emph{reference distribution} $\calD_p$ on $\Delta_K$ that we use to generate samples $\{p_i\}_{i=1}^m$ to feed into a regression method. In addition, for any $f \in \calH$, we later use the notation $\calD_{f(x)}$ to refer to the distribution on $\Delta_K$ induced by the marginal distribution $\calD_x$ of $x \in \calX$. 

\begin{assumption}[High-probability Error Bound]
\label{assump: high-prob-err-bd}
Let $\mathcal{G}$ be a family of candidate approximate optimal solution mappings whereby $\tilde{w}_\rho(\cdot) : \Delta_K \to S$ for all $\tilde{w}_\rho(\cdot) \in \mathcal{G}$. 
For each $j = 1, \ldots, K$, there exists a function $\calE^{\mathrm{prob}}_j(\cdot, \cdot; \mathcal{G}) : \mathbb{N} \times [0, 1) \to [0, \infty)$ such that, for any distribution $\calD_p$ on $\Delta_K$ and for any $\delta \in (0, 1]$, with probability at least $1 - \delta$ over $m$ independent samples drawn from $\calD_p$ with empirical distribution $\hat{\calD}^m_p$, it holds for all $\tilde{w}_\rho(\cdot) \in \mathcal{G}$ that:
\begin{equation*}
    \label{eq: high-prob-approx-bd}
    \left|\bbE_{\calD_p}[|w_{\rho,j}(p) -\tilde{w}_{\rho,j}(p)|] - \bbE_{\hat{\calD}^m_p}[|w_{\rho,j}(p) -\tilde{w}_{\rho,j}(p)|]\right| ~\leq~ \calE^{\mathrm{prob}}_j(m, \delta; \mathcal{G}).
\end{equation*}
\end{assumption}

When using an approximate optimal solution mapping with either a uniform or a high-probability error bound guarantee, a natural question is:  do the performance guarantees of the ICEO approach developed in Section \ref{sec: general-guarantee} extend to problem \eqref{program: approx-regularized}? We now answer this question affirmatively by extending the generalization bounds of Theorem \ref{thm:finite-sample-bd} to situations with approximate mappings satisfying either Assumption \ref{assump: uniform-err-bd} or Assumption \ref{assump: high-prob-err-bd}. We make an implicit assumption that, after solving problem \eqref{program: approx-regularized}, the decision-maker uses the correct optimal solution mapping $w_\rho(\cdot)$ to make decisions. Therefore, the left hand side of our bounds involve the true risk $R(w_{\rho} \circ f)$ with the correct mapping while the right hand sides involve the empirical risk $\hat{R}_n(\tilde{w}_{\rho} \circ f)$ with the approximation (and the reguarlized version thereof).
\begin{theorem}
\label{thm: gen-bd-approx} 
Suppose Assumption \ref{assump:c-lip-phi-constant} holds and that the hypothesis class $\calH$ has bounded multi-variate Rademacher complexity $\mathfrak{R}_n(\mathcal{H})$.
Then, for any $\delta \in (0,1]$ and $\rho_n > 0$, we have the following:
\begin{itemize}
    \item[(i)] If the approximate optimal solution mapping $\tilde{w}_\rho(\cdot)$ satisfies the uniform error bound as stated in Assumption \ref{assump: uniform-err-bd}, then with probability at least $1 - \delta$ over i.i.d. data $\{(x_i, \xi_i)\}_{i=1}^n$ drawn from the distribution $\calD$, the following inequalities hold for all $f \in \calH$:
\begin{align}
       R(w_{\rho_n}\circ f) &\leq \hat{R}_n(\tilde{w}_{\rho_n}\circ f ) +
        \frac{\sqrt{2}L_c^2}{\rho_n}\mathfrak{R}_n(\mathcal{H})
        +\bar{c} \sqrt{\frac{\log(\frac{1}{\delta})}{2n}}
    + L_c\sum_{j=1}^d \calE^\mathrm{unif}_j  \label{eq: approx-general-bd-unif1}
    \end{align}
    \item[(ii)] If the approximate optimal solution mapping $\tilde{w}_\rho(\cdot)$ comes from a family $\mathcal{G}$ satisfying the high probability error bound as stated in Assumption \ref{assump: high-prob-err-bd}, then with probability at least $1 - \delta$ over i.i.d. data $\{(x_i, \xi_i)\}_{i=1}^n$ drawn from the distribution $\calD$ and over $m$ independent samples $\{p_i\}_{i=1}^m$ drawn from a reference distribution $\calD_p$ on $\Delta_K$, the following inequalities hold for all $f \in \calH$:
    {\small
    \begin{align}
         R(w_{\rho_n} \circ  f)& \leq  \hat{R}_n(\tilde{w}_{\rho_n} \circ f) + L_c \sum_{j=1}^d\left[ \frac{1}{m} \sum_{i=1}^m  |w_{{\rho_n},j}(p_i) - \tilde{w}_{{\rho_n},j}(p_i)|+ \calE^{\mathrm{prob}}_j(n, \delta/2d; \mathcal{G}) + \calE^{\mathrm{prob}}_j(m, \delta/2d; \mathcal{G}) \right] \nonumber\\
        & \qquad +\frac{\sqrt{2}L_c^2}{\rho_n}\mathfrak{R}_n(\mathcal{H})
   +\mathrm{diam}(S)L_c \mathrm{TV}(\calD_{f(x)}, \calD_{p}) + \bar{c} \sqrt{\frac{\log(\frac{1}{\delta})}{2n}}
      \label{eq: approx-general-bd-prob1}
    \end{align}
    }%
\end{itemize}
\end{theorem}

\section{Numerical Experiments}\label{sec:experiments}
In this section, we demonstrate the numerical performance of our proposed ICEO framework using synthetic data. We first summarize the benchmark methods that we adopted for comparison:
\begin{enumerate}
    \item Sample average approximation (SAA). In this benchmark, the decision-maker simply ignores the contextual features and minimizes the average of cost functions using the empirical distribution of the observations of the random parameter.
     \item The two-step estimate-then-optimize (ETO) method is based on cross-entropy loss (ETO-Entropy). In this benchmark, we estimate the hypothesis $f \in \calH$ using the cross-entropy loss function (a standard loss function for multi-class classification) instead of the downstream optimization goal. 
     \item The prescriptive method (PRES) proposed by \cite{bertsimas2020predictive}. We consider the KNN-based (PRES-KNN) and kernel-based (PRES-Kernel) variants.
\end{enumerate}
Moreover, we consider two different types of ICEO methods. 
\begin{enumerate}
\item Vanilla ICEO (ICEO). ICEO method that solves \eqref{program:origin}.
\item The ICEO method regularized by cross-entropy loss (ICEO-Entropy). This is the ICEO method incorporating an additional cross-entropy loss term as a regularization component in conjunction with the objective function described in \eqref{program:origin}. Adding the cross-entropy term as a regularization component aligns with practical intuition provided by existing literature \citep{kao2012directed, elmachtoub2022smart}.
\end{enumerate}

\paragraph{Data Generation Process.} The synthetic data is generated in the following manner. The features $x_i \in \mathbb{R}^p$ are generated independently following the multi-variate Gaussian distribution $x_i\sim N(0, MI_p)$ for some constant $M > 0$ and where $I_p$ is an identity matrix. Then, given $K$ scenarios $\Xi = \{\tilde{z}_1, \dots, \tilde{z}_K\}$, the corresponding conditional probability vector is generated according to a randomly initialized neural network, denoted as $f_{\text{NN}}^{\ast}$, composed with the softmax function. The softmax function is implemented by adding a softmax layer as the output layer. Subsequently, $\xi_i$ takes on the value of $\tilde{z}_k$ with a probability of $p^\ast_k(x) = \mathrm{soft}\circ f_{\text{NN}}^{\ast}$ for all $k = 1, \dots, K$.

\paragraph{Optimal Solution Mapping Approximation.} The optimal oracle is approximated using neural networks in the experiment. We first generate a data set $\{(p_i, w_i)\}_{i=1}^m$ by uniformly sampling $p_i$ from the simplex $\Delta_K$ and then generating $w_i := w_\rho(p_i)$. Then we train a neural network with one hidden layer to approximate the oracle. The neural network is trained with respect to the mean absolute percentage error (MAPE) loss. 

\paragraph{ICEO Hypothesis Learning.} In this experiment, we consider the hypothesis class $\calH := \mathrm{soft} \circ \tilde{\calH}$ where $\tilde{\calH}$ represents a neural network. If $\tilde{\calH}$ contains the true hypothesis $f_{\text{NN}}^{\ast}$, then $\mathcal{H}$ is well-specified. However, if $\tilde{\calH}$ only includes neural networks with fewer layers or hidden nodes, it indicates the case of model misspecification. For both cases of well-specification and misspecification, we employ the Adam optimization algorithm (\cite{kingma2014adam}) to solve \eqref{program: approx-regularized} and learn the hypothesis.

\subsection{Multi-item Newsvendor Problem}
\label{sec: numerical-newsvendor}
We consider the multi-item newsvendor problem, as in Example \ref{example:newsvendor}, with synthetic data. We consider $d=2$, which is the case where the newsvendor jointly decides the order quantities of two products with an overall budget of $50$. The decision variable $w\in \mathbb{R}^2$ and random demand $\xi \in \mathbb{R}^2$ are both two-dimensional, corresponding to the order quantity and demand of the two products. The newsvendor aims to minimize the total inventory cost as formulated in (\ref{eq: newsvendor-obj}), with unit overstock costs $h_1$ and $h_2$ set to $1$ and $1.3$ and unit stockout cost $b_1$ and $b_2$ set to $9$ and $8$ for the two products, respectively.

\paragraph{Results and Comparisons with Benchmarks.}
In this experiment, we vary the number of scenarios $K \in {5, 10, 15}$. The realizations for each scenario, ${\tilde{z}_1, \dots, \tilde{z}_K}$, are generated randomly. Further details on the scenario generation process can be found in Appendix \ref{appendix: exp}. We set the regularization coefficient $\rho=0.01$ and consider multiple training set sizes $n\in \{100, 300, 500, 700 \}$. For every value of $n$, we run 25 simulations. We use a validation set to tune the hyper-parameters for PRES-KNN, PRES-Kernel, ETO-Entropy and the ICEO method. 
To evaluate out-of-sample performances of all these methods, we generate a test set including 1000 samples in each simulation. We evaluate performance using the newsvendor cost \eqref{eq: newsvendor-obj}.

Figure \ref{fig:nv-iid-plots} illustrates the performance of the ICEO methods and the non-parametric benchmarks across different numbers of scenarios.  Each sub-figure demonstrates that the ICEO-Entropy method consistently outperforms all benchmarks, regardless of the training set sizes and numbers of scenarios.  Although the vanilla ICEO method sometimes falls short of the benchmarks except for SAA, incorporating the cross-entropy loss in ICEO leads to the superior performance of ICEO-Entropy, which outperforms all benchmarks, even the ETO entropy, in every scenario. A comparison between vanilla ICEO and ICEO-Entropy reveals that ICEO-Entropy outperforms the former. The possible reason might be that incorporating the oracle in the ICEO objective induces a higher level of non-convexity and increases the likelihood of converging to a poor local minimum since the ICEO objective is non-convex in the predicted distribution. By adding a small component of cross-entropy loss, the gradient-based method suffers less from non-convexity while still being influenced by the ICEO objective. Despite this computational issue, the ICEO objective demonstrates the advantages of considering the ultimate optimization goal, as evident when comparing ICEO-Entropy with ETO-Entropy. Moreover, when compared with non-parametric prescriptive methods, the superior performances of the two ICEO methods and the ETO-Entropy method highlight the benefit of modeling the underlying conditional distribution.
\begin{figure}[htbp]
    \centering
    \begin{subfigure}{0.48\textwidth}
        \centering
        \includegraphics[width=\textwidth]{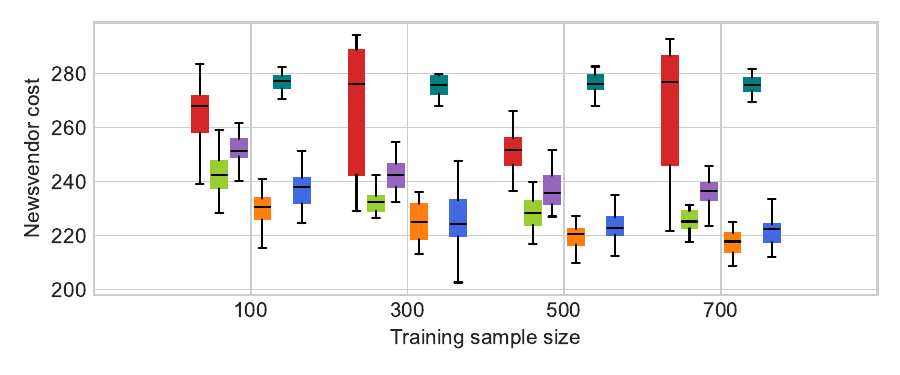}
        \caption{K = 5}
         \label{fig:iid-plot-k5}
    \end{subfigure}
    \begin{subfigure}{0.48\textwidth}
        \centering
        \includegraphics[width=\textwidth]{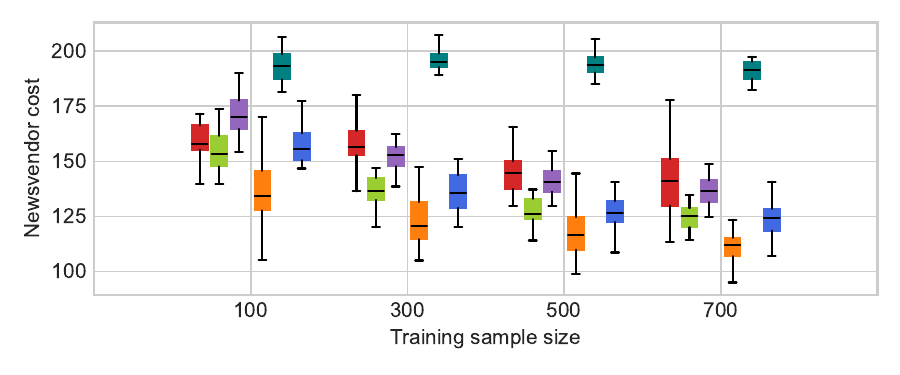}
        \caption{ K = 10}
         \label{fig:iid-plot-k10}
    \end{subfigure}

    \begin{subfigure}[c]{0.48\textwidth}
\includegraphics[width=\textwidth]{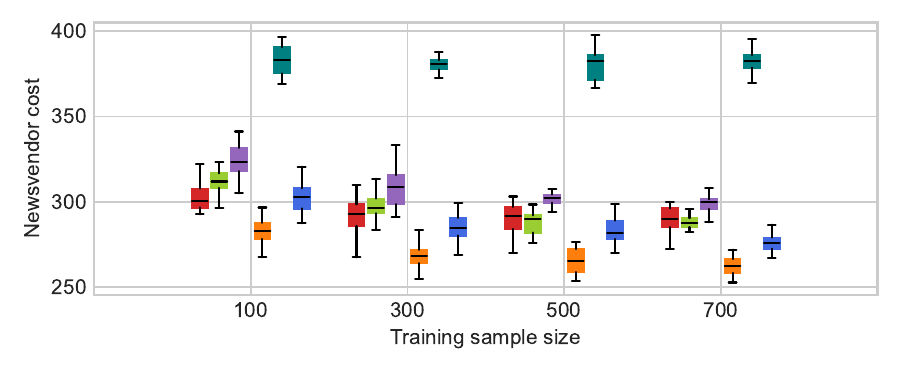}
        \caption{ K = 15}
         \label{fig:iid-plot-k15}
    \end{subfigure} 
  \begin{subfigure}[t]{0.39\textwidth}
\includegraphics[width=0.3\textwidth]{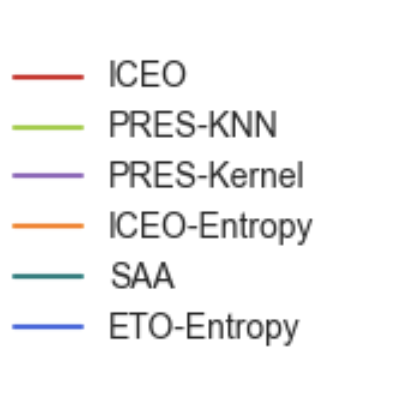}
    \end{subfigure}
    \caption{Comparison of ICEO with benchmark methods for multi-item newsvendor problem.}
    \label{fig:nv-iid-plots}
\end{figure}

\paragraph{Results on Model Misspecification}
We then examine the performance of ICEO and benchmark methods in the case of model misspecification. When the model is properly specified, the data generation neural network $f_{NN}^\ast$ consists of two hidden layers, whereas the hypothesis class $\mathcal{H}$ consists of linear models. Since the ICEO methods and ETO-Entropy are the only approaches that utilize this hypothesis class to model the underlying conditional distribution, we compare the performance of these three methods to study the effect of model misspecification. To evaluate their effectiveness, we utilize the newsvendor cost on a test set comprising 1000 samples in each simulation.
Figure \ref{fig:mismatch} illustrates the performance of vanilla ICEO and ICEO-Entropy in comparison to the two-step ETO-Entropy method. As we can see, under model misspecification, both ICEO methods consistently outperform the two-step ETO-Entropy approach.  This finding demonstrates the advantage of considering the ultimate optimization goal under model misspecification while estimating the conditional distribution.
\begin{figure}[ht]
\centering
  \includegraphics[width=0.6\textwidth]{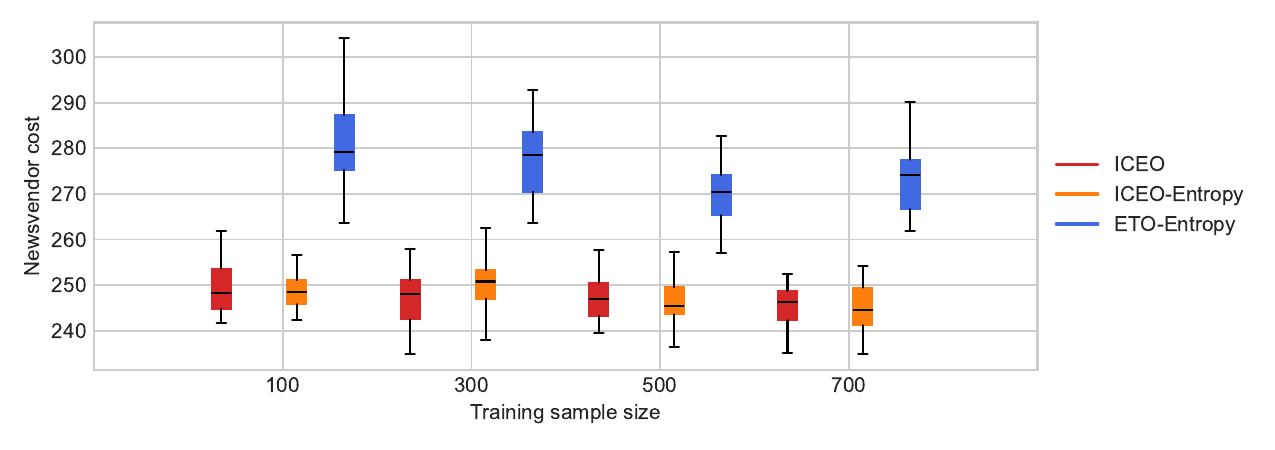} 
\caption{Comparison between ICEO and ETO-Entropy under model misspecification on the multi-item newsvendor problem.}
\label{fig:mismatch}
\end{figure}

\subsection{Quadratic Cost Network Flow Problem}
\label{sec: numerical-network}
In this part, we consider the minimum cost network flow problem. The problem formulation is as stated in Example \ref{example: networkflow}, where $g_d(w_d, \xi_d) = c_d(w_d-\xi_d)^2$. We consider a simple network demonstrated in Figure \ref{fig:flow-graph}, where there are two source nodes 1 and 2, and two sink nodes, 3 and 4. The amount of flow that sources out of each of the source nodes 1 and 2 must be no less than a threshold equal to $10$. Similarly, the amount of flow that goes into each of the sink nodes 3 and 4 must be at least 10. We let $w_1, w_2, w_3, w_4$ denote the amount of flow on arcs $(1,3), (1,4), (2,3), (2,4)$, respectively.
\begin{figure}[ht]
\centering
\includegraphics[width=0.2\textwidth]{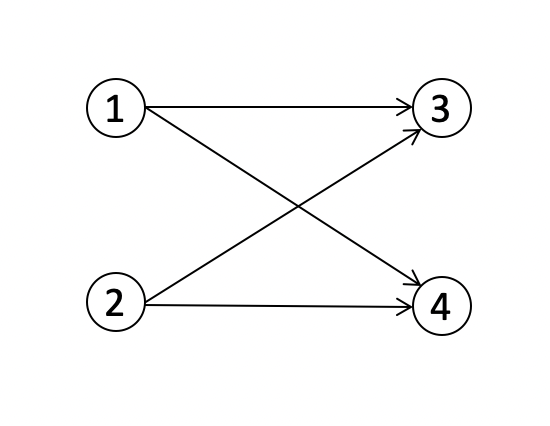} 
\caption{Network graph}
\label{fig:flow-graph}
\end{figure}
\paragraph{Results and Comparisons with Benchmarks.} In this experiment, we consider $K \in \{5, 10, 15\}$. The realizations of each scenario, $\tilde{z}_1, \dots, \tilde{z}_K$, are generated randomly. Again, more details regarding the scenario generation process can be found in Appendix \ref{appendix: exp}.
Furthermore, the weights of each arc $c_1, c_2, c_3, c_4$ take the values of $1,3,2,2$ respectively. The regularization coefficient $\rho = 0.01$. As in Section \ref{sec: numerical-newsvendor}, we consider multiple training set sizes $n\in \{100, 300, 500, 700\}$, we run 25 simulations for each sample size, and the test set includes 1000 samples in each simulation. To tune the hyper-parameters for ICEO, ETO-Entropy, and the two prescriptive methods, PRES-KNN and PRES-Kernel, we use a validation set including 1000 samples.
\begin{figure}[ht]
    \centering
    \begin{subfigure}{0.48\textwidth}
        \centering
\includegraphics[width=\textwidth]{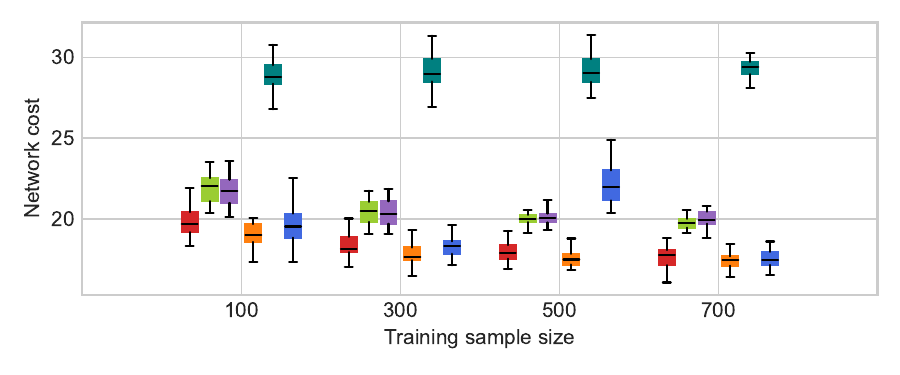}
        \caption{K = 5}
         \label{fig:nf-iid-plot-k5}
    \end{subfigure}
    \begin{subfigure}{0.48\textwidth}
        \centering
\includegraphics[width=\textwidth]{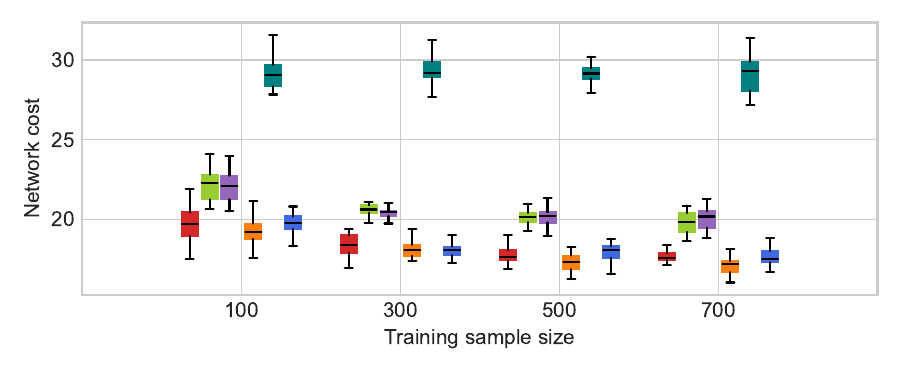}
        \caption{K= 10}
         \label{fig:nf-iid-plot-k10}
    \end{subfigure}
    \begin{subfigure}[c]{0.45\textwidth}
\includegraphics[width=\textwidth]{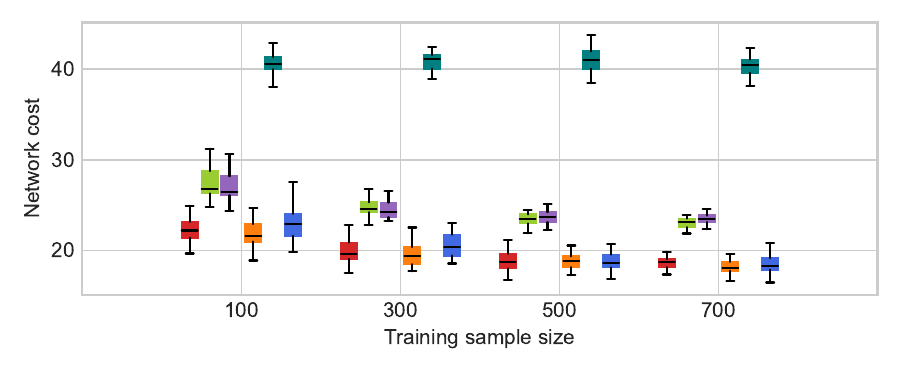}
        \caption{K = 15}
         \label{fig:nf-iid-plot-k15}
    \end{subfigure} 
  \begin{subfigure}[t]{0.39\textwidth}
\includegraphics[width=0.3\textwidth]{figures/legend.png}
    \end{subfigure}
    \caption{Comparison of ICEO with benchmark methods on the minimum cost network flow problem.}
    \label{fig:nf-iid-plots}
\end{figure}

Figure \ref{fig:nf-iid-plots} compare the test-set performance of the ICEO method and benchmarks across different numbers of scenarios and sample sizes. Each sub-figure demonstrates that either the vanilla ICEO method or the ICEO-Entropy method outperforms all benchmarks in this experiment. Compared to the best benchmark (PRES-KNN), there is a vague trend that the advantage of ICEO methods is more significant when the sample size is limited.

\paragraph{Results on Model Misspecification.} As in Section \ref{sec: numerical-newsvendor}, we investigate the case of model misspecification. The model misspecification is again introduced by having the data generation neural network $f^*_{NN}$ consist of two hidden layers, whereas the hypothesis class $\mathcal{H}$ consists of linear models. We only compare the two ICEO methods and ETO-Entropy because these are the only approaches that utilize this hypothesis class to model the underlying conditional distribution. Figure \ref{fig:mismatch} illustrates the performance of vanilla ICEO, ICEO-Entropy, and ETO-Entropy methods. It is shown that again, both vanilla ICEO and ICEO-Entropy consistently outperform ETO-Entropy. Moreover, one may observe a possible trend that this advantage of ICEO methods is stronger when there are fewer training samples. This finding again verifies the advantage of considering the ultimate optimization goal while estimating the conditional distribution.

\begin{figure}[ht]
\centering
  \includegraphics[width=0.65\textwidth]{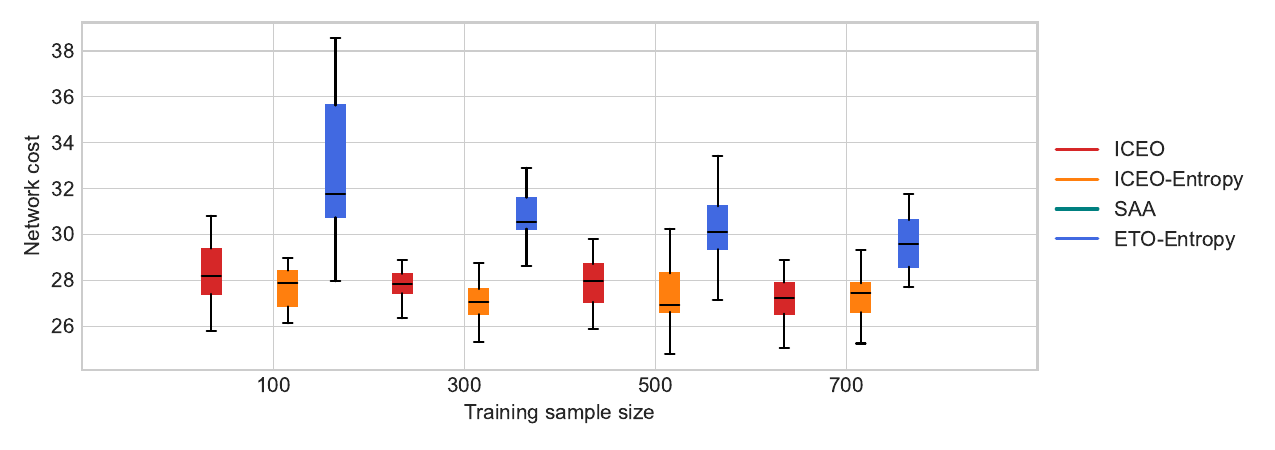} 
\caption{Comparison between ICEO and ETO-Entropy under model misspecification on the minimum cost network flow problem.}
\label{fig:mismatch-nf}
\end{figure}

\section{Conclusion}
In this paper, we propose a new framework for estimating the underlying conditional distribution in contextual stochastic optimization. The proposed ICEO framework uses a flexible hypothesis class to learn by incorporating the downstream optimization goal and applies readily to the case where the random parameter is a discrete random variable and the nominal optimization problem is convex. %
We then prove that the ICEO method is asymptotically consistent and provide finite-sample analysis in the form of generalization bounds. Moreover, we investigate the non-differentiability of the regularized optimal solution oracle which often leads to computational difficulties in calculating the gradients and poor local minima that are hard to escape. We address this issue by approximating the regularized oracle using differentiable functions. We then provide approximation error bounds and the corresponding generalization bounds when using the approximated oracle. 
Among others, a natural direction for future research is to move further beyond the assumption of finite suport for the random parameter by considering modeling extensions and/or theory.

\ACKNOWLEDGMENT{PG acknowledges the support of NSF AI Institute for Advances
in Optimization Award 211253.}

\bibliographystyle{informs2014} %

\bibliography{reference}

\newpage
\begin{APPENDICES}
\section{Lemmas and Proofs for Section \ref{sec:consistency}}
\label{sec: consistency-proof-appendix}

\begin{lemma}
\label{lemma: opt-cont-func}
For any $\rho >0$, $w_\rho(\cdot)$ is a single-valued continuous function of $p\in \Delta_K$.
\end{lemma}
\proof{Proof of Lemma \ref{lemma: opt-cont-func}}
From part 2 in Theorem 9.17 of \cite{sundaram1996first}, the mapping $w_\rho(p): \Delta_k \rightarrow S$ is a continuous function of $p$. %
\Halmos
\endproof
\begin{lemma}
\label{lemma:f-star-minimizer}
Suppose that Assumptions \ref{assump: specification} and \ref{assump: unique} hold, and let $w(\cdot) : \Delta_K \to S$ be an optimal solution mapping such that $w(p) \in W(p)$ for all $p\in \Delta_K$. Then, we have that:
(i) $f^\ast$ is the unique minimizer of $\min_{f\in \mathcal{H}} R(w \circ f; 0)$;
(ii) the optimal expected unregularized risk is $J^\ast :=\min_{f\in \mathcal{H}} R(w \circ f; 0)= R(w\circ f^\ast, 0)$, and the value of $J^\ast$ is the same for all valid optimal solution mappings $w(\cdot)$.
\end{lemma}

\proof{Proof of Lemma \ref{lemma:f-star-minimizer}}
We first prove part {\em (i)}. Suppose that there exists $\bar{f} \neq f^\ast$ in $\calH$ such that
\begin{equation*}
   \mathbb{E}_x \left[\sum_{k=1}^K p^\ast_k(x)c_k(w(\bar{f}(x))) \right]
     \leq \mathbb{E}_x \left[\sum_{k=1}^K p^\ast_k(x)c_k(w(f^*(x))) \right].
\end{equation*}
Since  $w(f^\ast(x)) \in W(f^\ast(x))$ and $f^\ast(x)$ is the true hypothesis, i.e., $f_k^\ast(x) = p_k^\ast(x)$, we have 
\begin{equation*}
   \sum_{k=1}^K p^\ast_k(x)c_k(w(\bar{f}(x))) \geq
     \sum_{k=1}^K p^\ast_k(x)c_k(w(f^*(x)))
\end{equation*}
for all $x\in \calX$. Combining the above two inequalities yields 
\begin{equation*}
   \sum_{k=1}^K p^\ast_k(x)c_k(w(\bar{f}(x))) =
     \sum_{k=1}^K p^\ast_k(x)c_k(w(f^*(x))),
\end{equation*}
almost surely for all $x\in \calX$. Hence, $w(\bar{f}(x)) \in W(f^\ast(x))$ almost surely, which contradicts Assumption \ref{assump: unique}.
Therefore, under Assumptions \ref{assump: specification} and \ref{assump: unique}, $f^*$ is the unique minimizer of $\min_{f\in \mathcal{H}} R(w \circ f; 0)$.

We now show {\em(ii)}. Let $w(\cdot), u(\cdot) : \Delta_K \to S$ be two valid optimal solution mappings such that $w(p) \in W(p)$ and $u(p) \in W(p)$ for all $p \in \Delta_K$. Thus, $w(f^\ast(x))\in W(f^\ast(x))$ and $u(f^\ast(x))\in W(f^\ast(x))$ for all $x$, which yields $\sum_{k=1}^K p^\ast_k(x)c_k(w(f^\ast(x)))$ =  $\sum_{k=1}^K p^\ast_k(x)c_k(u(f^\ast(x)))$ by the definition of $W(\cdot)$. Therefore, taking expectation with respect to $x$ yields $J^\ast = R(w\circ f^\ast, 0) = R(u\circ f^\ast, 0)$.
\Halmos
\endproof

\proof{Proof of Proposition \ref{prop: oracle-consistency}}
Let $p\in \Delta_K$ be given. We define the correspondence $g_p(\rho)$ by $g_p(\rho):= w_\rho(p)$ for $\rho > 0$ and $g_p(0):= W(p)$. Here, $g_p(\cdot)$ is a function that maps a value of $\rho \in [0, \infty)$ to a subset of $S$ (which is just a single point $w_\rho(p)$ when $\rho > 0$). Given the sequence $\{\rho_n> 0\}$ that converges to zero, define a corresponding sequence ${w_n}$ where $w_n = w_\rho(p) \in g_p(\rho_n)$.  If $\mathrm{dist}(w_{\rho_n}(p), W(p)) \not\rightarrow 0$ as $n \to \infty$, then there exists a constant $\epsilon >0$, such that for some subsequence $w_{n(m)}$ we have $\text{dist}(w_{n(m)}, W(p))> \epsilon$ for all $m$. By the Weierstrass Theorem and the compactness of $S$, assume without loss of generality that $w_{n(m)} \to \bar{w} \in S$ as $m \to \infty$. Due to the continuity of the objective function in \eqref{program:origin-lower-level_regularized} with respect to $\rho$, we apply the maximum theorem (\cite{berge1877topological}), which guarantees that $g_p(\rho)$ is upper hemicontinuous in $\rho$. Thus, since $\rho_{n(m)} \to 0$ as $m \to \infty$, we have that $\bar{w} \in g_p(0) = W(p)$ according to the definition of upper hemicontinuity.
This directly contradicts the claim that $\text{dist}(w_{n(m)}, W(p)) > \epsilon$ for all $m$.
\Halmos
\endproof

\begin{lemma}
\label{lemma: finite-cover-num}
Consider the class of functions $g : \mathcal{X}\times \Xi \to \mathbb{R}$ defined by $\mathcal{F}:= \{g : g = c \circ w_\rho \circ f \text{ for some } \rho \in (0, \rho_0] \text{ and } f\in \mathcal{H}\}$.
If Assumptions \ref{assump: finite-bracket-N} and \ref{assump: oracle-uni-lip-rho}-\ref{assump: oracle-uni-lip} hold, then for any $\delta \in (0,1)$, the $3\delta$-bracketing number $N(3\delta; \mathcal{F}, \|\cdot\|_\infty)$, with respect to the sup-norm $\|\cdot\|_\infty$, is finite.
\end{lemma}
\proof{Proof for Lemma \ref{lemma: finite-cover-num}}
Assumption \ref{assump: finite-bracket-N} guarantees that, for any $\delta\in (0,1)$, we can find a $\frac{\delta}{2L_cL_w}$-bracket of $\mathcal{H}$. That said, for each $f\in \mathcal{H}$, there exist an $i\in \{1, \dots, N_1\}$ such that $l^i_k(x) \leq f^i_k(x) \leq u^i_k(x)$ for all $k=1, \dots, K$ and $x\in \calX$ and with $\|l^i-u^i\|_\infty \leq \frac{\delta}{2L_cL_w}$.

By the uniform Lipschitzness of $w_\rho$ in $\rho$ (Assumption \ref{assump: oracle-uni-lip-rho}), we have that for any $\rho_1$, $\rho_2$ and $p$, $\|w_{\rho_1}(p) -w_{\rho_2}(p)\| \leq L_\rho |\rho_1-\rho_2|$. Therefore, $|c(w_{\rho_1}(p), \xi) - c(w_{\rho_2}(p), \xi)| \leq L_c L_\rho |\rho_1 -\rho_2| $.
Then we consider $T = \lfloor \frac{\rho_0 L_c L_w}{2\delta} \rfloor$ and have a collection of points $\{\rho^0, \dots, \rho^{T+1}\}$, where $\rho^i := \frac{\delta}{L_cL_\rho}i$, for $i=0, \dots, T$, and $\rho^{T+1} := \rho_0$.

Now we show that $\{[c\circ w_{\rho^{i}} \circ l^j -  \delta ,\quad  c\circ w_{\rho^{i}} \circ u^j +  \delta ] \quad : \quad  \forall i = 0, \dots, T+1, \forall j=1, \dots, N_1 \}$  forms a $3\delta$-cover of $\mathcal{F}$. We first show that for any $\rho\in (0, \rho_0]$ and $f\in \calH$, there exists $i$ and $j$ such that 
\begin{align*}
&\quad  |c(w_{\rho} (f(x)), \xi) -  c (w_{\rho^{i}}( u^j(x)), \xi) | \\
& \leq |c(w_{\rho} (f(x)), \xi) -  c (w_{\rho^{i}}( f(x)), \xi) | + |c(w_{\rho^{i}} (f(x)), \xi) -  c (w_{\rho^{i}}( u^j(x)), \xi) |\\
& \leq L_c L_\rho |\rho - \rho_i|  + L_c\|w_{\rho^i}(u^j(x)) -w_{\rho^i}(f(x)) \|\\
& \leq \frac{\delta}{2} + L_c L_w \|u^j(x) - f(x) \|\\
& \leq \frac{\delta}{2}  + \frac{\delta}{2}  = \delta.
\end{align*}
The second inequality holds by the Lipschitz Assumptions \ref{assump: oracle-uni-lip-rho} and the the fact that the objective function $c(\cdot, \xi)$ is $L_c$-Lipschitz for any $\xi$. Then the third inequality follows from Assumption
\ref{assump: oracle-uni-lip}. 
Therefore, $c(w_{\rho} (f(x)), \xi) \leq c (w_{\rho^{i}}( u^j(x)), \xi) + \delta$. The same reasoning leads to  the fact that $c(w_{\rho} (f(x)), \xi) \geq c (w_{\rho^{i}}( l^j(x)), \xi) - \delta$. 

Note that for any $x$ and $\xi$
\begin{align*}
& \quad |(c (w_{\rho^{i}}( u^j(x)), \xi) + \delta) -  (c (w_{\rho^{i}}( l^j(x)), \xi)-\delta) | \\
& \leq 2 \delta + |c (w_{\rho^{i}}( u^j(x)) -  c (w_{\rho^{i}}( l^j(x)), \xi)) | \\
& \leq 2\delta + L_cL_w \|u^j(x)- l^j(x) \|.
\end{align*}
Therefore, $\|(c\circ w_{\rho^{i}} \circ u^j +  \delta)  - (c\circ w_{\rho^{i}} \circ l^j) \| \leq 2\delta + 0.5\delta \leq 3\delta$ and we have fund a $3\delta$-bracket with a size of $N_1 \times T+1 < \infty$.
\Halmos
\endproof

\subsection{Justifying Assumptions for Theorem \ref{thm: consistency}}
In this subsection, we justify the validity of the uniform Lipschitzness of the regularized oracle (Assumptions \ref{assump: oracle-uni-lip-rho}-\ref{assump: oracle-uni-lip}).
To achieve that, we start by showing that the desired property can be guaranteed by a common ``automatic crossover'' phenomenon of regularized oracle $w_\rho(\cdot)$.
\begin{assumption}[Automatic Crossover of the Regularized Oracle]
\label{assump: jumping-oracle}
There exists a positive number $\rho_{c, \phi}$, so that for all $\rho\leq \rho_{c, \phi}$ and all $p\in \Delta_K$, $w_\rho(p) \in W^*(p)$. The universal constant $\rho_{c, \phi}$ only depends on the property of the objective function $c(\cdot, \cdot)$ and the regularization $\phi(\cdot)$. 
\end{assumption}
This assumption indicates that, if $\rho$ is smaller than a threshold $\rho_{c, \phi}$, the output of the regularized oracle is in the optimal solution set of the unregularized oracle.

Combining the automatic crossover assumption and Proposition \ref{prop:w-lip}, we have a universal Lipschitz constant for any $\rho \geq 0$.
\begin{corollary}
\label{cor: uni-lip-p} Suppose Assumption \ref{assump: jumping-oracle} holds, then
for all $\rho > 0$, and any $p_1$ and $p_2$ from $\Delta_K$, $\|w_{\rho}( p_1 ) - w_{\rho}( p_2 ) \| \leq \frac{L_c}{\rho_{c, \phi}}\|p_1-p_2\|_2$.
\end{corollary}
\proof{Proof of Corollary \ref{cor: uni-lip-p}}
Directly follows by combining Proposition \ref{prop:w-lip} and Assumption \ref{assump: jumping-oracle}.
\Halmos
\endproof 

The previous results further imply the uniform Lipschitzness with respect to the regularization parameter $\rho$ for any $p\in \Delta_K$. 
\begin{corollary}
\label{cor: uni-lip-rho}
Suppose Assumption \ref{assump: jumping-oracle} holds, then
for all $p$ from $\Delta_k$, and any $\rho_1$ and $\rho_2$ from $(0, \rho_0] $, $\|w_{\rho_1}(p) - w_{\rho_2}(p)\| \leq \frac{L_c\sqrt{K}}{\rho^2_{c, \phi}}|\rho_1-\rho_2|$.%
\end{corollary}
\proof{Proof of Corollary \ref{cor: uni-lip-rho}}
For any $p \in \Delta_K$, and any $\rho > 0$, $w_\rho(p)$ is also the unique minimizer of 
\begin{equation*}
\sum_{k=1}^K \frac{p_k}{\rho}c_K(w) + \phi(w).
\end{equation*}
For simplicity, we assume that $\rho_1 < \rho_2$.
Then since $\rho_1 \geq \rho_{c, \phi}$, we
 apply Proposition \ref{prop:w-lip}, we have
\begin{align*}
\|w_{\rho_1}(p) - w_{\rho_2}(p)\|  \leq L_c\left\|\frac{p}{\rho_1} - \frac{p}{\rho_2}\right\|_2 \leq L_c\|p\|_2\left|\frac{1}{\rho_1} - \frac{1}{\rho_2}\right|
 \leq L_c \frac{\sqrt{K}}{\rho_{c, \phi}^2} |\rho_1-\rho_2|.
\end{align*}
The first inequality holds by applying Proposition \ref{prop:w-lip} and the last inequality follows from the the fact that $\Delta_K$ is compact and Assumption \ref{assump: jumping-oracle}.
Similarly, if $\rho_1 < \rho_{c, \phi} \leq \rho_2$, the above analysis also applies because $w_{\rho_1}(p) = w_{\rho_{c, \phi}}(p)$ for any $p\in \Delta_K$ by Assumption \ref{assump: jumping-oracle} and that $|\rho_1 -\rho_2| \geq |\rho_{c, \phi} - \rho_2|$. In the last case where $\rho_1 < \rho_2 \leq \rho_{c, \phi}$, then $w_{\rho_1}(p) = w_{\rho_2}(p) = w_{\rho_{c, \phi}}(p)$ for any $p\in \Delta_K$. Thus the desired result still trivially holds.
\Halmos
\endproof

\subsubsection{A Sufficient Condition for the Automatic Crossover Property}
\label{sec: appendix-jumping}
In this part, we let $h(w, p) := \sum_{k=1}^K p_kc_k(w)$ for simplicity. $W^\ast(p) := \argmin_{w\in S} h(w, p)$ and $h^\ast( p) := \min_{w\in S} h(w, p) $. Then we demonstrate that the automatic crossover property is quite common, for example, it holds for any convex objective $c(\cdot, \cdot)$ as long as it satisfies the following linear growth condition stated in Assumption \ref{assump: linear-growth}.
\begin{assumption}[Linear Growth]
\label{assump: linear-growth}
There exists $\mu >0$ such that for any $p\in \Delta_K$ and $w\in S$, it holds that
$h(w, p) - h^\ast(p) \geq \mu \mathrm{dist}(w, W^\ast(p))$.
\end{assumption}
Note that, for example, any piecewise linear function satisfies the linear growth condition.

\begin{theorem}
\label{thm: jumping-oracle}
Suppose Assumption \ref{assump: linear-growth} holds. Then, for all $\rho \in (0, \rho_{c, \phi}]$, the regularized oracle $w_\rho(p)$ automatically crosses over (Assumption \ref{assump: jumping-oracle}) with the unregularized solution set $W^*(p)$, i.e., $w_\rho(p) \in W^*(p)$. The phase transitioning constant of the crossover is $\rho_{c, \phi} := \frac{\mu}{\overline{\nabla\phi}}$, where $\overline{\nabla\phi} := \sup_{w\in S} \|\nabla \phi(w) \|_* < \infty$.
\end{theorem}
\proof{Proof of Theorem \ref{thm: jumping-oracle}}
We let $\bar{w}_{\rho}(p):= \argmin_{w\in W^\ast(p)} \|w-w_\rho(p)\|$ denote the projection of $w_\rho(p)$ to the unregularized optimal solution set $W^\ast(p)$.
For any $p\in \Delta_K$ and any $\rho \geq 0$, we have
\begin{align*}
0 & \geq [h(w_\rho(p), p) + \rho \phi(w_\rho(p)) ]-  [h(\bar{w}_\rho(p), p) + \rho \phi(\bar{w}
_\rho(p)) ]\\
& = [h(w_\rho(p), p) - h(\bar{w}_\rho(p), p)] + \rho [ \phi(w_\rho(p)) - \phi(\bar{w}
_\rho(p))] \\
& = [h(w_\rho(p), p) - h(\bar{w}_\rho(p), p)] + \rho [ \phi(w_\rho(p)) - \phi(\bar{w}_\rho(p)) -\nabla \phi(\bar{w}_\rho(p))^T(w_\rho(p)-\bar{w}_\rho(p))] \\
& \quad + \rho\nabla \phi(\bar{w}_\rho(p))^T(w_\rho(p)-\bar{w}_\rho(p))\\
& \geq \mu \|w_\rho(p)-\bar{w}_\rho(p) \| + \frac{\rho}{2}\|w_\rho(p)-\bar{w}_\rho(p) \|^2 -\rho \|\nabla\phi(\bar{w}_\rho(p)) \|_\ast \|w_\rho(p)-\bar{w}_\rho(p) \| \\
&\geq  \|w_\rho(p)-\bar{w}_\rho(p) \| (\mu + \frac{\rho}{2}\|w_\rho(p)-\bar{w}_\rho(p) \| -\rho \overline{\nabla\phi}).
\end{align*}
The first inequality follows from the optimality of $w_\rho(p)$. The first term of the second inequality holds because of Assumption \ref{assump: linear-growth}, and the second term in the second inequality holds due to the 1-strong convexity of $\phi$. If $\rho \leq \frac{\mu}{\overline{\nabla\phi}}:= \rho_{c, \phi}$, then $\mu + \frac{\rho}{2}\|w_\rho(p)-\bar{w}_\rho(p) \| -\rho \overline{\nabla\phi} \geq 0$. Thus, $\|w_\rho(p)-\bar{w}_\rho(p) \|$ has to be zero; otherwise, the right side is a positive number and we have a contradiction.
\Halmos
\endproof

\subsection{Proof of Theorem \ref{thm: consistency}}
\proof{Proof of Theorem \ref{thm: consistency}.}
We first show that $\lim_{\rho\rightarrow 0} J^*_\rho = J^*$. Recall that $J^* = R(w\circ f^*)$ and $J^*_\rho = R(w_\rho \circ f^*_\rho)$. 
As defined at the beginning of Section 3,  the function $w(\cdot) : \Delta_K \to S$ arbitrarily selects a value from the optimal solution set given by $W(\cdot)$. Specifically, for any given $p \in \Delta_K$, $w(p)$ outputs an arbitrary value from the set $ W(p) = \argmin_{w\in S} \sum_{k=1}^K p_kc_k(w)$.

We first show that, for any sequence $\{\rho_n\}$ that converges to zero and for any given $x$,
\begin{equation*}
 \sum_{k=1}^K f^*_k(x)c_k(w_{\rho_n}(f^*(x))) \to \sum_{k=1}^K f^*_k(x)c_k(w(f^*(x))) \ \text{ point-wise as } n \to \infty.
\end{equation*}
Let $\epsilon > 0$ be fixed. 
Note that $c_k(\cdot)$ is Lipschitz and, therefore, uniformly continuous. Thus, there exists a $\delta > 0$ such that, for any $\bar{w}_n$, if $\|w_{\rho_n}(f^\ast(x)) - \bar{w}_n\| < \delta$, then $|c_k(w_{\rho_n}(f^\ast(x))) - c_k(\bar{w}_n)| < \epsilon$. By utilizing the convergence result provided by Proposition 1, there is a value of $N$ such that for all $n > N$, it holds that $\text{dist}(w_{\rho_n}(f^\ast(x)), W(f^\ast(x))) < \delta$. Then we let $\bar{w}_n := \argmin_{w \in W(f^\ast(x))} \|w_{\rho_n}(f^\ast(x)) - w\|$. It is important to note that for any $\bar{w}_1, \bar{w}_2 \in W(f^\ast(x))$, we have $\sum_{k=1}^K f^\ast_k(x)c_k(\bar{w}_1) = \sum_{k=1}^K f^*_k(x)c_k(\bar{w}_2)$. Therefore, for all $n>N$, the following holds
\begin{equation*}
|\sum_{k=1}^K f^*_k(x)c_k(w_{\rho_n}(f^*(x)))-\sum_{k=1}^K f^*_k(x)c_k(w(f^*(x)))| < \epsilon,
\end{equation*}
where again note that $w(\cdot) : \Delta_K \to S$ arbitrarily selects a value from the optimal solution set.
This shows that $\sum_{k=1}^K f^*_k(x)c_k(w_{\rho_n}(f^*(x)))$ converges to $\sum_{k=1}^K f^*_k(x)c_k(w(f^*(x)))$ point-wise as $n \to \infty$.

Then, by the dominated convergence theorem we have that
\begin{align}
R(w_{\rho_n} \circ f^*) & :=  \mathbb{E}_x \left[\sum_{k=1}^K f^*_k(x)c_k(w_{\rho_n}(f^*(x)))\right] \nonumber \\ \to R(w \circ f^*) & :=  \mathbb{E}_x \left[\sum_{k=1}^K f^*_k(x)c_k(w(f^*(x)))\right] \ \text{ as } n \to \infty. \label{eq: dc_converge}
\end{align}
Note that
\begin{align*}
R(w_{\rho_n} \circ f^*) & :=  \mathbb{E}_x \left[\sum_{k=1}^K f^*_k(x)c_k(w_{\rho_n}(f^*(x)))\right] \\
& \leq  \mathbb{E}_x \left[ \sum_{k=1}^K f^*_k(x)c_k(w_{\rho_n}(f^*(x))) + {\rho_n}\phi(w_{\rho_n}(f^*(x)))\right]\\
& \leq \mathbb{E}_x \left[ \sum_{k=1}^K f^*_k(x)c_k(w_{\rho_n}(f_{\rho_n}^*(x))) + {\rho_n} \phi(w_{\rho_n}(f_{\rho_n}^*(x))) \right]\\
& \leq  \mathbb{E}_x \left[ \sum_{k=1}^K f^*_k(x)c_k(w_{\rho_n}(f_{\rho_n}^*(x))) + {\rho_n} \bar{\phi}\right]\\
& =  R(w_{\rho_n}\circ f^\ast_{\rho_n}) + \rho_n \bar{\phi} \\
& \leq  \mathbb{E}_x \left[ \sum_{k=1}^K f^*_k(x)c_k(w_{\rho_n}(f^*(x)))\right] + {\rho_n}\bar{\phi}\\
& =  R(w_{\rho_n} \circ f^*) + \rho_n \bar{\phi}.
\end{align*}
The first inequality holds due to the presence of an additional non-negative regularization term. The second inequality arises from the definition of $w_{\rho_n}(\cdot)$. Specifically, $w_{\rho_n}(f^*(x))$ is the minimizer of $\min_{w\in S}\sum_{k=1}^K f^*_k(x)c_k(w) + {\rho_n}\phi(w) $ for any given $x$ while $w_{\rho_n}(f_{\rho_n}^*(x))$ is not. The third inequality holds because the decision regularization is upper-bounded by $\bar{\phi}$. Then the last inequality holds because $f^*_{\rho_n}$ is the minimizer of $\min_{f\in \mathcal{H}} \mathbb{E}_x \sum_{k=1}^K f^*_k(x)c_k(w_{\rho_n}(f(x))) $.  Notice that the above chain of inequalities demonstrates that $R(w_{\rho_n}\circ f^\ast_{\rho_n}) \in [R(w_{\rho_n} \circ f^*) - \rho_n \bar{\phi}, R(w_{\rho_n} \circ f^*) + \rho_n \bar{\phi}]$. 
Now let $\epsilon > 0$ again be fixed. Since $\rho_n \to 0$, there exists $N_1$ such that, for all $n> N_1$, $\rho_n < \frac{\epsilon}{2\bar{\phi}}$ and therefore $|R(w_{\rho_n}\circ f^\ast_{\rho_n}) - R(w_{\rho_n} \circ f^*)| \leq \epsilon/2$.
At the same time, by \eqref{eq: dc_converge}, there exists $N_2$ such that, for all $n> N_2$, we have $|R(w_{\rho_n} \circ f^*)- R(w \circ f^*)|< \frac{\epsilon}{2}$.
Therefore, considering $N= \max\{N_1, N_2\}$, we have that for all $n>N$,
\begin{align*}
|R(w_{\rho_n}\circ f^\ast_{\rho_n}) - R(w \circ f^*) | &\leq |R(w_{\rho_n}\circ f^\ast_{\rho_n}) - R(w_{\rho_n} \circ f^*) | + |R(w_{\rho_n} \circ f^*)-R(w \circ f^*)| \\
&\leq \frac{\epsilon}{2} + \frac{\epsilon}{2} = \epsilon.
\end{align*}
Therefore, as $\rho_n$ converges to zero, we have
\begin{equation*}
J^*_{\rho_n} :=  R(w_{\rho_n} \circ f^*_{\rho_n}) \rightarrow  R(w\circ f^*) = J^*.
\end{equation*}

If Assumptions \ref{assump: finite-bracket-N} and \ref{assump: oracle-uni-lip-rho}-\ref{assump: oracle-uni-lip} hold, Lemma \ref{lemma: finite-cover-num} shows that, the class of function $\mathcal{F}:= \{c \circ w_\rho \circ f: \mathcal{X}\times \Xi \to \mathbb{R}| \rho \in (0, \rho_0], f\in \mathcal{H} \}$ has a finite bracketing number. Therefore, we apply Theorem 3.2 in \cite{sen2018gentle}(see also \cite{wainwright2019high} and \cite{van2000asymptotic}) and have
\begin{equation}
\label{eq: uniform-converge-empirical-risk}
\hat{R}_n(w_\rho \circ f) \to R(w_\rho \circ f) \ \text{ with probability }1,
\end{equation}
 uniformly for $\rho \in (0, \rho_0]$ and $f\in \mathcal{H}$. 
Then by \eqref{eq: uniform-converge-empirical-risk}, we have
\begin{align*}
|\hat{J}^n_\rho - J^\ast_\rho| & = \max \left\{ \hat{R}_n(w_\rho \circ \hat{f}^n_\rho) - R(w_\rho \circ f^\ast_\rho), R(w_\rho \circ f^\ast_\rho) -\hat{R}_n(w_\rho \circ \hat{f}^n_\rho) \right\} \\
& \leq \max \left\{ \hat{R}_n(w_\rho \circ f^\ast_\rho) - R(w_\rho \circ f^\ast_\rho), R(w_\rho \circ \hat{f}^n_\rho) -\hat{R}_n(w_\rho \circ \hat{f}^n_\rho) \right\} \\
& \leq \max \left\{ |\hat{R}_n(w_\rho \circ f^\ast_\rho) - R(w_\rho \circ f^\ast_\rho)|, |R(w_\rho \circ \hat{f}^n_\rho) -\hat{R}_n(w_\rho \circ \hat{f}^n_\rho)| \right\},
\end{align*}
where the second inequality holds due to the fact that $\hat{R}_n(w_\rho \circ f^\ast_\rho) \geq \hat{R}_n(w_\rho \circ \hat{f}^n_\rho) $ and that $R(w_\rho \circ \hat{f}^n_\rho) \geq  R(w_\rho \circ f^\ast_\rho)$, because $\hat{f}^n_\rho$ minimizes $\hat{R}_n(w_\rho \circ f)$ and $f^\ast_\rho$ minimizes $R(w_\rho \circ f)$. 
By the uniform convergence of the empirical risk in $\rho\in (0, \rho_0]$ and $f\in \mathcal{H}$, guaranteed by \eqref{eq: uniform-converge-empirical-risk}, for any positive $\epsilon >0$, there exists $N$ such that for all $n>N$, $\sup_{\rho \in (0, \rho_0]}|\hat{R}_n(w_\rho \circ f^\ast_\rho) - R(w_\rho \circ f^\ast_\rho)| < \epsilon$ and $\sup_{\rho \in (0,\rho_0]}|R(w_\rho \circ \hat{f}^n_\rho) -  \hat{R}_n(w_\rho \circ \hat{f}^n_\rho)| < \epsilon$ with probability 1. Thus, we conclude that $\sup_{\rho \in (0, \rho_0]}|\hat{J}^n_\rho - J^\ast_\rho| \to 0$ with probability 1 as $n \to \infty$. Therefore, since $\{\rho_n\}_{n=1}^\infty$ satisfies $\rho_n \in (0, \rho_0]$ and $\rho_n \to 0$ as $n \to \infty$, we have that 
\begin{equation*}
    |\hat{J}^n_{\rho_n} - J^*| ~\leq~ |J^*_{\rho_n} - J^*| + \sup_{\rho \in (0, \rho_0]}|\hat{J}^n_\rho - J^\ast_\rho| ~\to~ 0 \text{ with probability 1, as } n \to \infty.
\end{equation*}

Now we prove {\em (ii)} and {\em (iii)} simultaneously. Note that
\begin{align}
 \mathrm{dist}(w_{\rho_n}(\hat{f}^n_{\rho_n}( x)), W(f^*(x) ) & = \inf_{u\in W(f^*(x)) } \|w_{\rho_n}(\hat{f}^n_{\rho_n}( x)) -w_{\rho_n}(f^*_{\rho_n} (x)) +w_{\rho_n}(f^*_{\rho_n} (x)) - u\|\nonumber \\
   &  \leq \|w_{\rho_n}(\hat{f}^n_{\rho_n}( x)) -w_{\rho_n}(f^*_{\rho_n} (x))\| + \inf_{u\in W(f^*(x)) } \| w_{\rho_n}(f^*_{\rho_n} (x)) - u\| \nonumber \\
   & \leq \|w_{\rho_n}(\hat{f}^n_{\rho_n}( x)) -w_{\rho_n}(f^*_{\rho_n} (x))\| +\mathrm{dist}(w_{\rho_{n}}(f^*_{\rho_{n}} (x)), W(f^*(x))). \label{eq: two-terms-w}
\end{align}
The first inequality follows from the triangle inequality. The second inequality holds by the definition of $\mathrm{dist}(\cdot, \cdot)$. 
In the remaining part of this proof, we separately establish the convergence in probability of the first and second terms above, both holding $\mathcal{D}_x$-almost surely over $x$.

 Due to the compactness of $S$, with any sequence of $\rho_n $ converging to zero, the sequence $w_{\rho_n}(f^*_{\rho_n} (x))$ has accumulation points. Let $w_{\rho_{n(m)}}(f^*_{\rho_{n(m)}} (x))$ be any subsequence converging to an accumulation point $\bar{w}(x) \in S$. 
Note that we have
\begin{align*}
    J^*  = \lim_{m\rightarrow \infty} \mathbb{E}_x\left[\sum_{k=1}^K f^*_k(x)c_k(w_{\rho_{n(m)}}(f^*_{\rho_{n(m)}}( x))) \right]
    \geq  \mathbb{E}_x\left[\sum_{k=1}^K f^*_k(x)c_k(\bar{w}(x))  \right],
\end{align*}
where the equality follows from the convergence of  $J^*_\rho  \rightarrow J^*$ when $\rho \rightarrow 0$, and the first inequality holds by Fatou's lemma. Therefore, $\mathcal{D}_x$ almost surely for all $x$, $\bar{w}(x)$ must lie in the set $W(f^*(x))$.
Then $\mathrm{dist}(w_{\rho_{n(m)}}(f^*_{\rho_{n(m)}} (x)), W(f^*(x))) \leq \|w_{\rho_{n(m)}}(f^*_{\rho_{n(m)}} (x)) - \bar{w}(x)\| = 0$. The aforementioned reasoning holds for every converging subsequence $n(m)$, then the original sequence $\mathrm{dist}(w_{\rho_{n}}(f^*_{\rho_{n}} (x)), W(f^*(x))) $ also converges to zero  $\mathcal{D}_x$-almost surely for all $x$.

Now we aim to demonstrate the convergence of $\hat{f}^n_\rho$ to $f^\ast_\rho$ for any fixed $\rho \in (0, \rho_0]$. Note that the first condition required by Theorem 5.7 in \cite{van2000asymptotic} is the uniform convergence in $f$, which is already established in \eqref{eq: uniform-converge-empirical-risk}. The second condition, which requires $f^*_\rho$ to be a well-separated point given each $\rho$, is implied by Assumption \ref{assump: unique} combined with Assumption \ref{assump: unique-decision}. Then we apply Theorem 5.7 in \cite{van2000asymptotic} and have that, for each $\rho \in (0, \rho_0]$, $\hat{f}_\rho^n$ convergences to $f^*_\rho$ in probability as $n \to \infty$. Due to the continuity of $w_\rho(\cdot)$, we have $\|w_\rho(\hat{f}_\rho^n(x)) - w_\rho(f^\ast_\rho(x)) \|$ converges to zero in probability for all $x$.

Then we want to show the following claim:  $\sup_{\rho \in (0, \rho_0]}\|w_{\rho}(f^*_{\rho}(x)) -w_{\rho}(\hat{f}^n_{\rho}( x)) \|$ converges to zero in probability, $\mathcal{D}_x$-almost surely for all $x$.
By the uniform convergence of $|\hat{R}(w_\rho\circ f) - R(w_\rho \circ f)|$ presented in \eqref{eq: uniform-converge-empirical-risk}, for any $\delta >0$,
there exists $N$, such that for any $n \geq N$,   $\sup_{\rho \in (0, \rho_0]}|\hat{R}(w_\rho\circ\hat{f}_\rho^n) - R(w_\rho\circ\hat{f}^n_\rho)| <\delta/2$.
If the desired claim is not true,  there exists $\epsilon >0$ and a subsequence $\{m(n)\}$ with $\mathbb{P}_x (\sup_{\rho \in (0, \rho_0]}\|w_{\rho}(f^*_{\rho}(x)) -w_{\rho}(\hat{f}^{m(n)}_{\rho}( x)) \| > \epsilon) >0$ for all $n = 1, \dots, \infty$. Then for each $n$, one may choose a $\rho$ that achieves $\mathbb{P}_x (\|w_{\rho}(f^*_{\rho}(x)) -w_{\rho}(\hat{f}^{m(n)}_{\rho}( x)) \| \geq \epsilon) >0$. Then by Assumption \ref{assump: unique-decision}, there must exist a $\delta >0$ such that $\sup_{\rho \in (0, \rho_0]}|R(w_\rho \circ f^*_{\rho}) - R(w_\rho \circ \hat{f}^{m(n)}_{\rho})| \geq\delta$. 
Then we notice that, for this subsequence $\{m(n)\}$, so that there exists a positive $\delta$ such that
\begin{align*}
& \quad \sup_{\rho\in (0, \rho_0]}|\hat{J}_\rho^{m(n)}- J^\ast_\rho|\\
& = \sup_{\rho\in (0, \rho_0]}|\hat{R}(w_\rho \circ \hat{f}^{m(n)}_{\rho}) - R(w_\rho \circ f^*_{\rho}) |\\ 
&\geq \sup_{\rho\in (0, \rho_0]}|R(w_\rho \circ \hat{f}^{m(n)}_{\rho}) - R(w_\rho \circ f^*_{\rho}) | - \sup_{\rho\in (0, \rho_0]}|\hat{R}(w_\rho(\hat{f}_\rho^{m(n)})) - R(w_\rho(\hat{f}^{m(n)}_\rho))|\\ & \geq \delta - \delta/2 = \delta/2 >0.
\end{align*}
This contradicts the uniform convergence of $\hat{J}^n_\rho$ to $J^*_\rho$ in probability.
Therefore, for any decreasing sequence $\{\rho_n\}_{n=1}^\infty$ with $\rho_n>0$ and $\rho_n \rightarrow 0$, we have $\|w_{\rho_n}(\hat{f}^n_{\rho_n}(x)) - w_{\rho_n}(f^\ast_{\rho_n}(x))\| \leq \sup_{\rho \in (0, \rho_0]} \|w_\rho(\hat{f}^n_\rho(x)) - w_\rho(f^\ast_\rho(x))\| ~\to~0$ with probability 1 for $\mathcal{D}_x$-almost surely. Then, returning to (\ref{eq: two-terms-w}), we have
\begin{align*}
    \mathrm{dist}(w_{\rho_n}(\hat{f}^n_{\rho_n}( x)), W(f^*(x) ) 
   & \leq \|w_{\rho_n}(\hat{f}^n_{\rho_n}( x)) -w_{\rho_n}(f^*_{\rho_n} (x))\| +\mathrm{dist}(w_{\rho_{n}}(f^*_{\rho_{n}} (x)), W(f^*(x))) \label{eq: asymp-eq2} \\ 
&\to 0, \qquad \mathcal{D}_x-\text{almost surely}, \text{with probability 1}.\nonumber
\end{align*}
Thus {\em (ii)} is proved.

Finally, if we have an accumulation point of $\hat{f}^n_{\rho_n}$, denoted as $\bar{f}$, we have 
\begin{align*}
\mathrm{dist}(w_{\rho_n}(\hat{f}^n_{\rho_n}(x)), W(\bar{f}(x))) & \leq \|w_{\rho_n}(\hat{f}^n_{\rho_n}(x))-  w_{\rho_n}(\bar{f}(x))\|+ \mathrm{dist}(w_{\rho_n}(\bar{f}(x)), W(\bar{f}(x))) \\
& \leq  L_w \|\hat{f}^n_{\rho_n}(x) - \bar{f}(x)\| + \mathrm{dist}(w_{\rho_n}(\bar{f}(x)), W(\bar{f}(x)))\\
& \to 0, \qquad \text{as $n\to \infty$ for all $x\in \mathcal{X}$}.
\end{align*}
Knowing that $\mathrm{dist}(w_{\rho_n}(\bar{f}(x)), W(\bar{f}(x))) \to 0$ holds by Proposition \ref{prop: oracle-consistency}.
Then by {\em (ii)}, $  \mathrm{dist}(w_{\rho_n}(\hat{f}^n_{\rho_n}( x)), W(f^\ast(x) ) \to 0$, $\mathcal{D}_x$-almost surely with probability 1, as $n\to \infty$. Thus, for any accumulation point of $w_{\rho_n}(\hat{f}^n_{\rho_n}(x))$, denoted by $\bar{w}$, it holds that $ \bar{w} \in W(\bar{f}(x))$ for all $x\in \calX$ and that  $\bar{w} \in W(f^\ast(x))$ $\mathcal{D}_x$-almost surely with probability 1.
That is, $W(\bar{f}(x)) \cap W(f^*(x)) \neq \emptyset$, $\mathcal{D}_x$-almost surely with probability 1. Then if $\bar{f} \neq f^\ast$, it contradicts with Assumption \ref{assump: unique}.
Thus with the uniqueness assumption, the true hypothesis $f^*$ can be recovered by $\hat{f}^n_{\rho_n}$.
\Halmos
\endproof

\newpage
\section{Supplementary Lemmas and Proofs for Sections \ref{sec:generalization} and \ref{sec: computational}}
\label{sec: proofs}

\subsection{Lemmas and Proofs for Section \ref{sec:generalization}}
\label{sec: appendix-general-guarantee}

\proof{Proof of Proposition \ref{prop:w-lip}}
Let $p, p^\prime \in \bbR_+^{K}$ be fixed. We let $h_\rho(\cdot, p) : S \to \bbR$ be defined by $h_\rho(w, p) :=\sum_{k=1}^K p_k c_k(w) +\rho\phi(w)$. Since $\phi(\cdot)$ is a $1$-strongly convex function, then $h_\rho(\cdot, p)$ is $\rho$-strongly convex and it holds for all $w \in S$ and $g\in \partial_w h_\rho(w, p)$ that
\begin{equation}
\label{eq: lip-1}
h_\rho(w^\prime, p) - h_\rho(w, p) ~\geq~ g^T(w^\prime - w) + \frac{\rho}{2}\|w^\prime - w\|^2 \qquad \forall w^\prime \in S.
\end{equation}
Since $w_\rho(p) = \arg\min_{w \in S}h_\rho(w, p)$, the first-order optimality condition implies there exists a subgradient $g\in \partial h(w_\rho(p), p)$ such that $g^T(w^\prime - w_\rho(p)) \geq 0$ for all $w^\prime \in S$. Applying this condition in \eqref{eq: lip-1} with $w \gets w_\rho(p)$ $w^\prime \gets w_\rho(p^\prime)$ yields
\begin{equation*}
h_\rho(w_\rho(p'), p) - h_\rho(w_\rho(p), p) ~\geq~ \frac{\rho}{2}\|w_\rho(p')-w_\rho(p)\|^2.
\end{equation*}
Switching the role of $p$ and $p'$ yields
\begin{equation*}
     h_\rho(w_\rho(p), p') - h_\rho(w_\rho(p'), p') ~\geq~ \frac{\rho}{2}\|w_\rho(p)-w_\rho(p')\|^2.
\end{equation*}
Adding the above two inequalities together yields
\begin{align*}
     \rho \|w_\rho(p)-w_\rho(p')\|^2 ~&\leq~ h_\rho(w_\rho(p), p') - h_\rho(w_\rho(p'), p') + h_\rho(w_\rho(p'), p) - h_\rho(w_\rho(p), p)\\
     ~&=~ [h_\rho(w_\rho(p), p') - h_\rho(w_\rho(p), p)] - [h_\rho(w_\rho(p'), p') - h_\rho(w_\rho(p'), p)]\\
     ~&=~\sum_{k=1}^K (p'_k-p_k) c_k(w_\rho(p)) - \sum_{k=1}^K (p'_k-p_k) c_k(w_\rho(p')) \\
     ~&=~ \sum_{k=1}^K (p'_k-p_k) (c_k(w_\rho(p)) - c_k(w_\rho(p'))) \\
     ~&\leq~  \|p - p^\prime\|_2\|c(w_\rho(p)) - c(w_\rho(p^\prime))\|_2 \\
     ~&\leq~  L_{c}\|p-p'\|_2\|w_\rho(p)- w_\rho(p')\|,
\end{align*}
where the last inequality uses Assumption \eqref{assump:c-lip}. Dividing by $\|w_\rho(p)- w_\rho(p')\|$ leads to \eqref{eq: w-lip}, and combining the resulting inequality again with \eqref{assump:c-lip} yields \eqref{eq: c-eq-lip}.
\Halmos
\endproof

\proof{Proof of Theorem \ref{thm:finite-sample-bd}}
Due to Proposition \ref{prop:w-lip}, in particular, the Lipschitz property of $c(\cdot)$ in \eqref{eq: c-eq-lip}, we can apply the vector contraction inequality from \cite{maurer2016vector} which, stated in terms of empirical Rademacher complexities, yields
\begin{equation*}
\hat{\mathfrak{R}}_n(c \circ w_{\rho_n}\circ \mathcal{H}) ~\leq~ \frac{\sqrt{2}L_c^2}{\rho_n}\hat{\mathfrak{R}}_n(\mathcal{H}).
\end{equation*}
Taking expectations of both sides of the above inequality, with respect to i.i.d. data $\{(x_i, \xi_i)\}_{i=1}^n$ drawn from the distribution $\calD$, yields
\begin{equation*}
    \mathfrak{R}_n(c \circ w_{\rho_n}\circ \mathcal{H}) \leq \frac{\sqrt{2}L_c^2}{\rho_n}\mathfrak{R}_n(\mathcal{H}).
\end{equation*} 
Then, a direct application of Theorem \ref{thm: bartlett} yields the desired result. 
\Halmos
\endproof

\proof{Proof of Corollary \ref{cor: optimal-out-of-sample-bound}}
According to Theorem \ref{thm:finite-sample-bd}, we have 
\begin{align*}
    R(w_{\rho_n}\circ \hat{f}^n_{\rho_n}) &~\leq~ \hat{R}_n(w_{\rho_n} \circ \hat{f}^n_{\rho_n}) +
\frac{\sqrt{2}L_c^2}{\rho_n}\mathfrak{R}_n(\mathcal{H})
+\bar{c}\sqrt{\frac{\log(\frac{2}{\delta})}{2n}}
\end{align*}
with probability at least $1- \frac{\delta}{2}$. 
Then we apply Hoeffding's inequality and have
\begin{equation*}
    \hat{R}_n(w_{\rho_n} \circ f^\ast_\mathcal{H}) \leq R(w_{\rho_n} \circ f^\ast_\mathcal{H}) + \bar{c} \sqrt{\frac{2\log(\frac{2}{\delta})}{n}} \label{eq: hoeffding}
\end{equation*}
with probability at least $1- \frac{\delta}{2}$. Note that
\begin{equation*}
\hat{R}_n(w_{\rho_n} \circ \hat{f}^n_{\rho_n}) \leq \hat{R}_n(w_{\rho_n} \circ f^\ast_\mathcal{H}),
\end{equation*}
then we have
\begin{align*}
    R(w_{\rho_n} \circ \hat{f}^n_{\rho_n}) & \leq R(w_{\rho_n} \circ f_\mathcal{H}^\ast) +\frac{\sqrt{2}L_c^2}{\rho_n}\mathfrak{R}_n(\mathcal{H})
+\frac{3\bar{c}}{2}\sqrt{\frac{2\log(\frac{2}{\delta})}{n}}
\end{align*}
with probability at least $1- \frac{\delta}{2}$. If the hypothesis class $\mathcal{H}$ is well-specified, then $f^\ast_\mathcal{H}$ is $f^*$ by Lemma \ref{lemma:f-star-minimizer} and the same result holds for $f^\ast$.  
\Halmos
\endproof

\subsection{Proofs and Supplementary Results for Section \ref{sec: computational}}
\label{sec: appendix-computation}
\begin{example}[Linear nominal optimization problem]
\label{exp: non-linear}Consider an example with a linear objective function in the optimization stage, i.e., $c_j(w)$ is a linear function $c_j^Tw$ for some $c_j \in \bbR^d$ for all $j = 1,\dots, K$. Suppose we use the decision regularization function $\phi(w) := \tfrac{1}{2}\|w\|_2^2$.
For any $p \in \Delta_K$, let $\bar{c}(p) := \sum_{j = 1}^K p_jc_j$. Then, note that
\begin{equation*}
w_\rho(p) = \argmin_{w \in S}\left\{\bar{c}(p)^Tw + \tfrac{\rho}{2}\|w\|_2^2\right\} = \argmin_{w \in S}\left\{\tfrac{\rho}{2}\|(\bar{c}(p)/\rho) - w\|_2^2\right\} = \Pi_S(\bar{c}(p)/\rho),
\end{equation*}
where $\Pi_S(\cdot)$ is the Euclidean projection operator onto $S$. Then, (\ref{program:origin}) is the problem of minimizing a sum of linear functions composed with projection operators, which is generally non-convex. At best, when $S$ is a polyhedron, i.e., $S := \{w \in \mathbb{R}^d: Aw \leq b\}$ and when we adopt a linear hypothesis class $\mathcal{H}= \{x \mapsto Bx \in \Delta_K : B \in \bbR^{K \times p}\}$, we can formulate (\ref{program:origin}) as a bilinear quadratic optimization problem.
Indeed, (\ref{program:origin}) can be reformulated as 
\begin{align}
\tag{ICEO-$\rho$-LP}
\label{program: lp-example}
  \min_{B, w_i, \lambda_i} \quad & \frac{1}{n}\sum_{i =1}^n\sum_{j=1}^K \mathbbm{1}\{\xi_i = \tilde{z}_j\} (Bx_i)_j c_j^Tw_i  \nonumber\\
\mathrm{s.t.} \quad & \frac{\rho}{2} w_i^Tw_i +\sum_{j = 1}^K
(Bx_i)_jc_j^Tw_i +\frac{1}{2} (\sum_{j = 1}^K (Bx_i)_jc_j+A^T\lambda)^T(\sum_{j = 1}^K (Bx_i)_jc_j+A^T\lambda)) \nonumber\\
&+ \lambda_i^Tb \leq 0, \quad \forall i = 1, \dots, n \nonumber\\
&Aw_i \leq b, \quad \forall i = 1, \dots, n\nonumber\\
& \lambda_i \geq 0 , \quad \forall i = 1, \dots, n\nonumber
\nonumber
\end{align}
Note that the dual function of the nominal quadratic optimization is $$-\frac{1}{2} ( \sum_{j = 1}^K (B^Tx_i)_jc_j+A^T\lambda)^T( \sum_{j = 1}^K (B^Tx_i)_jc_j+A^T\lambda)) - \lambda ^T b$$ thus the dual problem becomes 
\begin{equation*}
    \min_{\lambda \geq 0}  \frac{1}{2} ( \sum_{j = 1}^K (B^Tx_i)_jc_j+A^T\lambda)^T( \sum_{j = 1}^K (B^Tx_i)_jc_j+A^T\lambda)) + \lambda ^T b.
\end{equation*}
The first two group of constraints in (\ref{program: lp-example}) is to guarantee that $w_i$ and $\lambda_i$ are the optimal primal and dual solutions. The second and third group of constraints are for the primal and dual feasibility.
\Halmos
\end{example}

\proof{Proof of Proposition \ref{prop: krr-approx-error}}
Considering the kernel function $\mathcal{K}(p,p') = (c+p^Tp')^{s}$ with $p\in \mathbb{R}^K$, we first generalize the result of Example 13.19 from \cite{wainwright2019high}. When the input $p$ and $p'$ are $K$-dimensional vectors, the empirical kernel matrix can have rank at most $\frac{(s-1+K)!}{(s-1)!K!}$. Therefore, the left-hand side of Inequality (13.56) from \cite{wainwright2019high} can be upper-bounded by $\delta_m\sqrt{\frac{1}{m}\frac{(s-1+K)!}{(s-1)!K!}}$. Then we can apply Theorem 13.17 from \cite{wainwright2019high} and set $\lambda_m = 2\delta_m^2$ to achieve inequality (\ref{eq: poly-kernel-estimator1}). 
Moreover, the empirical Rademacher complexity can be upper-bounded by $\bar{c} \sqrt{\frac{1}{m} \frac{(s-1+K)!}{(s-1)!K!}}$ with some constant $\bar{c}$. Then if we have $\theta_m \geq \bar{c}_3 b \sqrt{\frac{1}{m} \frac{(s-1+K)!}{(s-1)!K!}}$, we can apply Theorem 14.1 from \cite{wainwright2019high} and therefore have the desired result inequality (\ref{eq: poly-kernel-estimator2}).
\Halmos
\endproof

\proof{Proof of Corollary \ref{cor: kernel}}
The proof follow from a slight modification of the proof of Theorem \ref{thm: gen-bd-approx}. We first consider
\begin{align*}
\frac{1}{n} \sum_{i=1}^n  \|w_\rho(f(x_i))) - \tilde{w}_\rho(f(x_i))\|_1 
  & \leq L_c  \sum_{j =1}^d (\frac{1}{n} \sum_{i=1}^n  |w_{\rho,j}(f(x_i)) - \tilde{w}_{\rho,j}(f(x_i))|^2)^\frac{1}{2}.
\end{align*} 
Then noted that
\begin{align*}
       \bbE_{ \calD_{f(x)}}[|w_{\rho,j}(p)) - \tilde{w}_{\rho,j}(p)|^2] \leq  \bbE_{ \calD_{p}}[|w_{\rho,j}(p)) - \tilde{w}_{\rho,j}(p)|^2] + 2\bar{w}^2_j \mathrm{TV}(\calD_{f(x)}, \calD_{p}),
 \end{align*}
 because $|w_{\rho,j}(p)) - \tilde{w}_{\rho,j}(p)|^2$ is bounded by $\bar{w}_j^2$ for all $p\in \Delta_K$.  
 Thus,
 \begin{align*}
       \bbE_{ \calD_{f(x)}}[|w_{\rho,j}(p)) - \tilde{w}_{\rho,j}(p)|^2]^{1/2} \leq  \bbE_{ \calD_{p}}[|w_{\rho,j}(p)) - \tilde{w}_{\rho,j}(p)|^2]^{1/2} + \bar{w}_j \sqrt{2\mathrm{TV}(\calD_{f(x)}, \calD_{p})}.
 \end{align*}
 Following the same reasoning of the proof in Theorem \ref{thm: gen-bd-approx}, the desired result follows. \Halmos 
\endproof

\proof{Proof of Theorem \ref{thm: gen-bd-approx}}
By Theorem \ref{thm:finite-sample-bd} (i), for any $\rho$, we have
\begin{align*}
    R(w_\rho \circ  f) \leq \hat{R}_n(w_\rho \circ  f) + \frac{\sqrt{2}L_c^2}{\rho}\mathfrak{R}_n(\mathcal{H})+\bar{c}\sqrt{\frac{\log(\frac{1}{\delta})}{2n}}.
\end{align*}
Noted that 
\begin{align}
  \frac{1}{n} \sum_{i=1}^n |c(w_\rho(f(x_i)),\xi_i) - c(\tilde{w}_\rho(f(x_i)),\xi_i) | & \leq L_c  \frac{1}{n} \sum_{i=1}^n  \|w_\rho(f(x_i))) - \tilde{w}_\rho(f(x_i))\|_1 \nonumber\\
 & = L_c  \sum_{j =1}^d \frac{1}{n} \sum_{i=1}^n  |w_{\rho,j}(f(x_i)) - \tilde{w}_{\rho,j}(f(x_i))| \label{eq: approx-gb-pf-bern1} %
\end{align}
When the approximated oracle has a uniform error, we combine (\ref{eq: approx-gb-pf-bern1}) and Assumption \ref{assump: uniform-err-bd} to have
 \begin{equation}
      \frac{1}{n} \sum_{i=1}^n  |c(w_\rho(f(x_i)),\xi_i) - c(\tilde{w}_\rho(f(x_i)),\xi_i) |   \leq L_c \sum_{j=1}^d \calE^\mathrm{unif}_j .
      \label{eq: approx-gb-pf-4 }
 \end{equation}
When the oracle is noised, we consider two different distributions $\calD_{f(x)}$ and $\calD_p$. We let $\calD_{f(x)}$ denote the distribution of $f(x)$ given a hypothesis $f$ and the distribution of $x$, $\calD_x$. Moreover, we let $\calD_p$ denote the distribution used to generate training samples $\{(p_i, w_i)\}_{i=1}^m$ for oracle approximation.  Then, we apply the error bound \eqref{eq: high-prob-approx-bd} with distribution $\calD_{f(x)}$ and $\calD_p$ respectively and have 
 \begin{align*}
      \frac{1}{n} \sum_{i=1}^n  |w_{\rho,j}(f(x_i)) - \tilde{w}_{\rho,j}(f(x_i))| \leq \bbE_{ \calD_{f(x)}}[|w_{\rho,j}(p)) - \tilde{w}_{\rho,j}(p)|]+ \calE^{\mathrm{prob}}_j(n, \delta/2d; \mathcal{G}), 
 \end{align*}
 and 
\begin{align*}
     \bbE_{ \calD_{p}}[|w_{\rho, j}(p)) - \tilde{w}_{\rho, j}(p)|] \leq  \frac{1}{m} \sum_{i=1}^m  |w_{\rho, j}(p_i)) - \tilde{w}_{\rho, j}(p_i)|  + \calE^{\mathrm{prob}}_j(m, \delta/2d; \mathcal{G}), %
 \end{align*} 
each with probability at least $1-\frac{\delta}{2d}$.
Considering the total variation between $\calD_{f(x)}$ and $ \calD_{p}$, we have the following
 \begin{align*}
       \bbE_{ \calD_{f(x)}}[|w_{\rho,j}(p)) - \tilde{w}_{\rho,j}(p)|] \leq  \bbE_{ \calD_{p}}[|w_{\rho,j}(p)) - \tilde{w}_{\rho,j}(p)|] + \mathrm{diam}_j (S) \mathrm{TV}(\calD_{f(x)}, \calD_{p}),
 \end{align*}
 where $\mathrm{diam}_j (S)$ is the diameter of S in the $j$-th coordinate and $\mathrm{TV}$ denotes the total variation. This result holds because $|w_{\rho,j}(\cdot) - \tilde{w}_{\rho,j}(\cdot)|$ is continuous and bounded by $\mathrm{diam}_j (S)$.
 Thus,  %
 {\small
 \begin{align*}
  \frac{1}{n} \sum_{i=1}^n  \|w_\rho(f(x_i))) - \tilde{w}_\rho(f(x_i))\|_1  \leq &  \sum_{j=1}^d\left[ \frac{1}{m} \sum_{i=1}^m  |w_{\rho,j}(p_i) - \tilde{w}_{\rho,j}(p_i)|+ \calE^{\mathrm{prob}}_j(n, \delta/2d; \mathcal{G}) + \calE^{\mathrm{prob}}_j(m/2d, \delta; \mathcal{G}) \right] \\
     &\qquad  +\mathrm{diam}_j (S) \mathrm{TV}(\calD_{f(x)}, \calD_{p}), %
 \end{align*}
 }%
 with probability at least $1-\delta$.
 Therefore, we have
 {\small
 \begin{align}
     \eqref{eq: approx-gb-pf-bern1} &\leq L_c \sum_{j=1}^d\left[ \frac{1}{m} \sum_{i=1}^m  |w_{i,j} - \tilde{w}_{\rho,j}(p_i)|+ \calE^{\mathrm{prob}}_j(n, \delta/2d; \mathcal{G}) + \calE^{\mathrm{prob}}_j(m, \delta/2d; \mathcal{G}) \right]
   +\mathrm{diam}(S) L_c \mathrm{TV}(\calD_{f(x)}, \calD_{p}),
\label{eq: error-term-prob}
 \end{align}
 }%
 with probability at least $1-\delta$. Here we slightly abuse the notation and let $\mathrm{diam}(S)$ denote the summation of coordinate-wise diameter along all coordinates.
 Then \eqref{eq: approx-general-bd-unif1} and \eqref{eq: approx-general-bd-prob1} follow from combining \eqref{eq: approx-gb-pf-4 } and \eqref{eq: error-term-prob} with Theorem \ref{thm:finite-sample-bd}. 
\Halmos
\endproof

\newpage
\section{Approximating the Optimal Solution Oracle by Polynomials}
\label{sec: approximation-polynomials}
In this section, we provide two examples of using polynomial functions to approximate the optimal solution mapping:  {\em (i)} interpolation using Bernstein polynominals, which satisfies a uniform error bound, and {\em (ii)} polynomial kernel regression, which satisfies a high-probability error bound. 

\subsection{Bernstein Polynomials}
\label{sec: bernstein}

One example for approximating the optimal solution mapping is interpolation using Bernstein polynomials, for which we review the definition below.
\begin{definition}[{\bf Bernstein Approximation (\cite{de2008complexity})}]\label{def:bern}
For a given function $\omega : \Delta_K \to \bbR$, the Bernstein approximation with order $s$, $B_s(\omega) : \Delta_K \to \bbR$, is defined by:
\begin{equation*}
\label{eq:bernstein}
    B_s(\omega)(p) ~:=~ \sum_{\alpha \in I(K,s)} \bar{w}\left(\frac{\alpha}{s}\right) \frac{s!}{\alpha !}p^\alpha, \quad \forall p\in \Delta_K,
\end{equation*}
where $I(K,s):= \{\alpha \in \mathbb{N}_0^K | \sum_{i=1}^K \alpha_i = s\}$, $\alpha! := \Pi_{i}\alpha_i!$, and $p^\alpha := p_1^{\alpha_1} \cdots p_K^{\alpha_K}$. \Halmos 
\end{definition}
Using Bernstein polynomials, based on a result of \cite{de2008complexity}, we can achieve a uniform bound of the approximation error as described in Assumption \ref{assump: uniform-err-bd}. 
\begin{proposition}
\label{prop: bern-approx-error}
For a given $\rho >0$, suppose that we use the Bernstein approximation method (Definition \ref{def:bern}) applied separately to each coordinate function $w_{\rho,j}(\cdot)$ to construct an approximate optimal solution mapping $\tilde{w}_\rho(\cdot)$. Then, $\tilde{w}_\rho(\cdot)$ satisfies the uniform error bound in Assumption \ref{assump: uniform-err-bd} with $\calE^\mathrm{unif}_j ~=~ \frac{\Omega L_{c}}{\rho\sqrt{s}}$,where $\Omega > 0$ is an absolute constant.
\end{proposition}
\proof{Proof of Proposition \ref{prop: bern-approx-error}}
This result directly follows from  Theorem 3.2 in \cite{de2008complexity} together with the Lipschitz property from Proposition \ref{prop:w-lip}.
\Halmos 
\endproof

Given the result in Proposition \ref{prop: bern-approx-error}, we can immediately obtain a generalization bound for the Bernstein approximation method by applying item {\em (i)} of Theorem \ref{thm: gen-bd-approx}. While the Bernstein polynomial method provides a strong uniform error bound guarantee, there is a significant drawback in the number of samples required to obtain this bound. Indeed, to accomplish this approximation, it involves knowing function values of $w_{\rho,j}(\cdot)$ on the grid $\Delta_{K,s} := \{w\in \Delta_K : sw \in \mathbb{N}_0^{K}\}$ which has $\binom{K+s}{K}$ many points in total. As such, the number of calculations of $w_{\rho}(\cdot)$ may be prohibitively large, which motivates the use of regression methods.

\subsection{Polynomial Kernel Regression}
\label{sec: krr}
In this section, we consider using the less computationally prohibitive regression methods that lead to high-probability bounds as in Assumption \ref{assump: high-prob-err-bd}. As an exemplary case, we consider the polynomial kernel regression method. In this setting, we allow for the possibility of a ``noised oracle" whereby the optimal solution mapping is not computed exactly. Specifically, the noised oracle outputs $w_\rho(p)+ \sigma \varepsilon$ instead of $w_\rho(p)$, where $\varepsilon$ is a $d$-dimensional standard Gaussian random vector and $\sigma$ is a scalar that represents the standard deviation of the noise. The approximate oracle $\tilde{w}_\rho(\cdot)$ is constructed on independent samples $\{(p_i, w_i)\}_{i = 1}^m$ where $p_i$ is drawn from the reference distribution $\calD_p$ and $w_i$ is computed from the noised oracle. That is, we assume that $w_i = w_\rho(p_i) + \sigma \varepsilon_i$ for $\{\varepsilon_i\}_{i=1}^m$ that are i.i.d. realizations of Gaussian random variables. These samples can be achieved by first generating $\{p_i\}_{i=1}^m$ randomly from following any user-chosen distribution $\calD_p$ over the simplex $\Delta_K$ and then calculating $\{w_i\}_{i=1}^m$ from a (possibly randomized) algorithm for approximating $w_\rho(\cdot)$. Note that we do assume that the noise is Gaussian, which may be a reasonable assumption for some algorithmic schemes for approximating $w_\rho(\cdot)$.

The approximate optimal solution mapping is learned using polynomial kernels $k(p,p')= (c+p^Tp')^s$, where $s\in \mathbb{N}$ is the degree parameter. In the remaining part of this section, we let $\mathcal{G}$ denote a function class induced by a polynomial kernel of degree $s$ and let $\|\cdot\|_\mathcal{G}$ denote any norm defined on $\mathcal{G}$. 
Note that $\mathcal{G}$ is a convex, star-shaped function class \cite{wainwright2019high}.
For the function class $\mathcal{G}$ and a given sample $\{p_i\}_{i = 1}^m$, let $\tau_j(\mathcal{G},\{p_i\}_{i = 1}^m, r):= \inf_{u\in \mathcal{G}: \|u\|_{\mathcal{G}} \leq r}(\frac{1}{m}\sum_{i=1}^m(u(\p_i) - w_{\rho,j}(p_i))^2)^{1/2} $ denote the fitting ability for $w_{\rho,j}$ using the kernel function class $\mathcal{G}$ within a user-defined radius $r$. %
Given the function class $\mathcal{G}$ and a given sample $\{(p_i,w_i)\}_{i = 1}^m$, the method of kernel ridge regression estimates the approximate optimal solution mapping $\tilde{w}_\rho(\cdot)$ by solving:
\begin{equation}
\label{eq:kernel_ridge}
\min_{u \in \mathcal{G}: \|u\|_{\mathcal{G}} \leq r} \frac{1}{m}\sum_{i=1}^m(u(\p_i) - w_{i,j})^2
\end{equation}

The corresponding high-probability approximation error bound for learning the noised oracle using polynomial kernel ridge regression. 
\begin{proposition}
\label{prop: krr-approx-error}
Let $\mathcal{G}$ denote a function class induced by a polynomial kernel of degree $s$, suppose that the noise of the output has standard deviation $\sigma$, and that we construct the approximate solution mapping $\tilde{w}_\rho(\cdot)$ using kernel ridge regression \eqref{eq:kernel_ridge} with a user-defined radius $r > 0$. Then, there exist absolute constants $\bar{c}, \bar{c}^\prime$ such that for all $\delta_m \geq \bar{c}_0 \frac{\sigma}{r}\frac{(s-1+K)!}{(s-1)!K!}\frac{1}{m}$, we have 
\begin{equation*}
    \frac{1}{m}\sum_{i=1}^m(\tilde{w}_{\rho, j}(\p_i) - w_{\rho,j}(\p_i))^2 
    ~\leq~ \bar{c}_1(\bar{c}'_1\tau_j(\mathcal{G},\{p_i\}_{i = 1}^m, r) + r^2 \delta_m^2),
    \label{eq: poly-kernel-estimator1}
\end{equation*}
with probability at least $1-\bar{c}_2 \exp(-\bar{c}'_2 \frac{mr^2}{\sigma^2} \delta_m^2 )$ for each coordinate $j = 1, \ldots, K$. Moreover, for any $\theta_m$ that satisfies $\theta_m \geq \bar{c}_3\sqrt{\frac{1}{m} \frac{(s-1+K)!}{(s-1)!K!}}$, if it also holds that $m\theta_m^2 \geq \bar{c}_0 \log(4\log(\frac{1}{\theta_m}))$, then
\begin{equation*}
    \left|\mathbb{E}_{\p}[(\tilde{w}_{\rho, j}(\p) - w_{\rho, j}(\p))^2 ]^{1/2}-\left(\frac{1}{m}\sum_{i=1}^m(\tilde{w}_{\rho,j}(\p_i) - w_{\rho,j}(\p_i))^2 \right)^{1/2}\right| ~\leq~ \bar{c}_3 r^2 \theta_m
    \label{eq: poly-kernel-estimator2}
\end{equation*}
with probability at least $1-\bar{c}_4\exp(-\bar{c}'_4 \frac{m\theta^2_m}{r^2} )$ for each coordinate $j = 1, \ldots, K$.
\end{proposition}
The main body of the proof is a generalization of the result in Example 13.19 of \cite{wainwright2019high}.
We have the corresponding generalization bound in the following corollary.
\begin{corollary}
\label{cor: kernel}
Suppose Assumption \ref{assump:c-lip-phi-constant} holds and that the hypothesis class $\calH$ has bounded multi-variate Rademacher complexity $\mathfrak{R}_n(\mathcal{H})$. Suppose further that we employ kernel ridge regression \eqref{eq:kernel_ridge} using a function class $\mathcal{G}$ induced by a polynomial kernel of degree $s$ under the same conditions as in Proposition \ref{prop: krr-approx-error}.
Then, for any $\delta \in (0,1]$ and $\rho_n > 0$, the following inequalities hold for all $f \in \calH$:
\begin{align*}
         R(w_{\rho_n} \circ  f)& ~\leq~  \hat{R}_n(\tilde{w}_{\rho_n} \circ f)
   + L_c  \sum_{j=1}^d[\bar{c}_3 r^2 (\theta_m + \theta_n)
      + \tau_j(\mathcal{G},\{p_i\}_{i = 1}^m, r) + \bar{c}'_1 r \delta_m]\nonumber \\ & \qquad \qquad  +\frac{\sqrt{2}L_c^2}{\rho_n}\mathfrak{R}_n(\mathcal{H})  +L_c\bar{w}_j \sqrt{2\mathrm{TV}(\calD_{f(x)}, \calD_{p})}+ \bar{c} \sqrt{\frac{\log(\frac{1}{\delta})}{2n}}
    \end{align*}
    with probability at least $1- \delta'$ over i.i.d. data $\{(x_i, \xi_i)\}_{i=1}^n$ drawn from the distribution $\calD$ and over $m$ independent samples $\{(p_i,w_i)\}_{i=1}^m$, where $\delta' = \delta +\bar{c}_2 \exp(-\bar{c}'_2 \frac{mr^2}{\sigma^2} \delta_m^2 )+ \bar{c}_4(\exp(-\bar{c}'_4 \frac{m\theta^2_m}{r^2} )+ \exp(-\bar{c}'_4 \frac{n\theta^2_n}{r^2} ))$ and $\delta_m, \theta_m$, and $\theta_n$ are chosen to satisfy the conditions in Proposition \ref{prop: krr-approx-error}.
\end{corollary}

Finally, in Section \ref{sec: semi-algebraic} we provide an alternative and exact computational approach in the semi-algebraic case when we use a polynomial function to approximate the optimal oracle.
\subsection{Computational Methods for the Semi-Algebraic Case}
\label{sec: semi-algebraic}
In this section, we present an approach based on polynomial optimization in the case where the objective of the downstream optimization problem is semi-algebraic and we use a linear hypothesis class. In this case, when we additionally use a polynomial approximation $\tilde{w}_\rho(\cdot)$, we can reformulate the approximate ICEO formulation \eqref{program: approx-regularized} as a polynomial optimization problem, which can be solved with a hierarchy of semi-definite optimizaiton problems.
Specifically we assume that both $c$ and $\phi$ are semi-algebraic functions and we consider the linear hypothesis class $\mathcal{H} = \{f(x): f(x) = Bx + b, (B,b)\in \mathcal{B}\}$ where $\mathcal{B}(\mathcal{X}) = \{(B,b) \in \mathbb{R}^{K \times p} \times \bbR^K : f(x) \in \Delta_K \ \forall x \in \calX\}$ ensures that the output of the hypothesis returns a feasible probability vector. In this section, we demonstrate an exact solution method for the semi-algebraic case by transforming the (\ref{program: approx-regularized}) to a polynomial optimization program. 
Before we reach the reformulated problem, we first review the definitions of semi-algebraic sets and semi-algebraic functions.
\begin{definition}[\bf Semi-algebraic Set (\cite{lasserre2015introduction})] $K\subset \mathbb{R}^n$ is a basic semi-algebraic set if
\begin{equation*}
    K = \{x\in \mathbb{R}^n : g_j(x) \geq 0, j = 1, \dots, m\}
\end{equation*}
for some polynomial functions $(g_j)_{j=1}^m$, i.e., $(g_j)_{j=1}^m \subset \mathbb{R}[x]$, where $\mathbb{R}[x]$ denotes the ring of real polynomials.
\end{definition}
Similarly, a semi-algebraic set is defined by not only a finite sequence of polynomial inequalities, but also equations, or finite union of these. 
\begin{definition}[\bf Semi-algebraic Function (\cite{lasserre2015introduction}).]
Let $K$ be a semi-algebraic set of $\bbR^n$. A function $f: K\rightarrow \bbR$ is a semi-algebraic function if its graph $\Psi_f := \{(x, f(x)): x\in K\}$ is a semi-algebraic set of $\bbR^n \times \bbR$.
\end{definition}
Functions generated by finitely many of dyadic operations $\{+, \times,\div, \vee,\wedge \}$ and monadic operations $|\cdot|$ and $(\cdot)^{1/q}$, $q\in \mathbb{N}$, on polynomials are semi-algebraic.

Note that with the linear hypothesis class, we need an additional constraint $Bx +b \in \Delta_K$, to guarantee that the output $f(x)$ is a valid probability vector for any $x \in \calX$. We also assume that $\mathcal{X}$ is a polyhedron, i.e. $\mathcal{X}:= \{x\in \mathbb{R}^p : Ax\geq a\}$ for some $A \in \mathbb{R}^{m\times p}$ and $a \in \bbR^m$.
Then the problem \ref{program: approx-regularized} becomes:
\begin{align}
\label{program:poly-approx-orig}
\tag{Poly-Approx-\text{ICEO}-$\rho_n$}
    \min_{B,b }\quad & \frac{1}{n} \sum_{i = 1}^n c(w_i, \xi_i) \\
    \mathrm{s.t.} \quad & w_i =  \tilde{w}_\rho(Bx_i + b) \nonumber \\
    & Bx+b \in \Delta_K, \forall x\in \mathcal{X} = \{x\in \mathbb{R}^p : Ax\geq a\} \nonumber
\end{align}
Note that the approximated oracle $\tilde{w}_\rho(\cdot)$ is constructed by polynomial kernels, so the first group of constraints are polynomial functions. Then we show that the second group of constraints can be reformulated to a group of linear constraints using the following proposition.
 \begin{proposition}Suppose $\mathcal{X}:= \{x\in \mathbb{R}^p : Ax\geq a\}$ for some $A \in \mathbb{R}^{m\times p}$ and $a \in \bbR^m$, then
 the constraint $$ Bx+b \in \Delta_K, \forall x\in \mathcal{X}$$  can be rewritten as the following group of constraints by introducing new decision variables $y_k \in \mathbb{R}^m, k = 1, \dots, K$, $z,u\in \mathbb{R}^m$
 \begin{align*}
 \begin{cases}
     & a^Ty_k \geq -b_{k} \quad \forall k = 1, \dots, K \\
     &A^Ty_k = B_k  \quad \forall k = 1, \dots, K\\
    & a^Tz \geq 1-\mathds{1}^Tb \\
    & A^Tz = B^T\mathds{1}\\
    & a^Tu \geq -1+ \mathds{1}^Tb  \\
    & A^Tu = -B^T\mathds{1}\\
    & y_k,z,u\geq 0 \quad \forall k = 1, \dots, K
    \end{cases}
 \end{align*}
 \label{prop: simplex-constr}
 \end{proposition}
 \proof{Proof of Proposition \ref{prop: simplex-constr}:}
The condition of 
\begin{equation*}
    Bx+b \in \Delta_K
\end{equation*}
can be represented by the following constraints:
\begin{align}
    &B_k^Tx+b_k \geq 0, \qquad \forall k =1, \dots, K \label{eq: non-negative} \\
    &\mathds{1}^T(Bx +b) \geq 1 \label{eq: normal-1} \\
    & \mathds{1}^T(Bx +b) \leq 1 \label{eq: normal-2} 
\end{align}
\eqref{eq: non-negative} represents the non-negativity constraints, while \eqref{eq: normal-1} and \eqref{eq: normal-2} consist of the normalization constraint. We first rewrite the non-negativity constraint for a component $k$, for all $x$ such that $Ax\geq a$ as
\begin{align*}
     0\leq \min \quad & B_k^Tx + b_{k}\\
    \mathrm{ s.t.} \quad & Ax\geq a. 
 \end{align*}
 Then consider the dual problem of the above linear programming and we have:
 \begin{align*}
   -b_{k}\leq   \max \quad & a^Ty_k \\
     \mathrm{s.t.} \quad & A^Ty_k = B_k \\
     & y_k \geq 0.
 \end{align*}
 which reduces to find a feasible solution of the following group of constraints
 \begin{align}
 \begin{cases}
 \label{dual-constraints-1}
      & a^Ty_k \geq -b_{k} \\
     &A^Ty_k = B_k  \\
     & y_k \geq 0. \\
 \end{cases}
 \end{align}
 Then we rewrite \eqref{eq: normal-1} for all $x\in \calX$ as 
 \begin{align*}
     1\leq \min \quad & \mathds{1}^TBx+ \mathds{1}^Tb \\
    \mathrm{ s.t.} \quad & Ax\geq a, 
 \end{align*}
 similarly by considering the dual problem
 \begin{align*}
     1 -\mathds{1}^Tb\leq \max \quad & a^Tz \\
     \mathrm{s.t.} \quad & A^Tz = B^T\mathds{1} \\
     & z \geq 0,
 \end{align*}
 which reduces to the following group of constraints
 \begin{align}
 \begin{cases}
 \label{dual-constraints-2}
  & a^Tz \geq 1-\mathds{1}^Tb \\
    & A^Tz = B^T\mathds{1}\\
    & z\geq 0 .\\
 \end{cases}
 \end{align}
 Finally, we consider the constraint \eqref{eq: normal-2}
 \begin{align*}
     1\geq \max \quad & \mathds{1}^TBx +\mathds{1}^Tb\\
    \mathrm{ s.t.} \quad & -Ax\leq -a 
 \end{align*}
 by strong duality, it is equivalent to 
 \begin{align*}
     1- \mathds{1}^Tb \geq \min \quad & -a^Tu \\
     \mathrm{s.t.} \quad & -A^Tu = B^T\mathds{1} \\
     & u \geq 0
 \end{align*}
 which reduces to 
  \begin{align}
 \begin{cases}
  \label{dual-constraints-3}
  & a^Tu \geq -1 + \mathds{1}^Tb \\
    & A^Tu = -B^T\mathds{1}\\
    & u\geq 0.
  \end{cases}
 \end{align}
 Therefore, the condition $Bx +b \in \Delta_K, \forall x\in \calX$ can be represented by combining \eqref{dual-constraints-1}, \eqref{dual-constraints-2}, and \eqref{dual-constraints-3}.
 \Halmos
\endproof
We have now shown that problem \eqref{program:poly-approx-orig} is a problem optimizing a basic semi-algebraic function on a basic semi-algebraic set which, by Proposition 11.10 of \cite{lasserre2015introduction}, can be reformulated as a polynomial optimization problem, which can be solved by solving a hierarchy of semi-definite problems.

\newpage
\section{Supplementary Materials for Section \ref{sec:experiments}}
\label{appendix: exp}
\paragraph{Demand generation (multi-product newsvendor).} For $K = \{5, 10, 15\}$, we generate scenarios $\{\tilde{z}_1, \dots, \tilde{z}_K\}$ randomly in the following manner. First, we randomly generate an aggregated demand vector in the dimension of $K=15$, where each component follows a Normal distribution with a mean of 70 and a standard deviation of 15. Then, for each $k$, we split the aggregated demand between two products using the weights $u_k$ and $1-u_k$, where $u_k$ follows a uniform distribution $U[0,1]$. We generated the scenarios once and used the demand scenarios listed in Table \ref{table: scenarios} for all numerical experiments in Section \ref{sec: numerical-newsvendor}.
\begin{table}[ht]
\centering
\begin{tabular}{c|cc}
\hline
Scenarios & Product 1 & Product 2 \\
\hline
1 &15.080 & 76.154 \\
2 & 47.238 & 6.483 \\
3 & 4.120 & 56.635 \\
4 & 21.646 & 37.001 \\
5 & 8.686 & 66.519 \\
6 & 8.989 & 84.359 \\
7 & 2.149 & 23.506 \\
8 & 7.857 & 82.529 \\
9 & 1.774 & 77.951 \\
10 & 57.281 & 17.009 \\
11 & 5.628 & 108.726 \\
12 & 40.191 & 46.200 \\
13 & 82.417 & 5.146 \\
14 & 19.817 & 71.711 \\
15 & 24.311 & 42.974 \\
\hline
\end{tabular}
\caption{Randomly generated demand scenarios for multi-product newsvendor problem}
\label{table: scenarios}
\end{table}

\paragraph{Demand generation (quadratic minimum cost network flow).} For $K = \{5, 10, 15\}$, we generate scenarios $\{\tilde{z}_1, \dots, \tilde{z}_K\}$ randomly in the following manner. First, we randomly generate an aggregated vector that its $k$-th element represents $\sum_{d=1}^4 \tilde{z}_{k, d}$. This is the sum of $\xi_d$ over all edges. Each component of this aggregated vector follows a Normal distribution with a mean of 15 and a standard deviation of 5. Then, for each $k$, we split the aggregated demand across four edges using a weight vector randomly sampled from the $K$-th simplex $\Delta_K$ following the Dirichlet distribution. We generated the scenarios once and used the demand scenarios listed in Table \ref{table: nf-scenarios} for all numerical experiments in Section \ref{sec: numerical-network}.
\begin{table}[ht]
\centering
\begin{tabular}{c|c c c c}
\hline
Scenarios & Edge 1 & Edge 2 & Edge 3 & Edge 4\\
\hline
1 & 1.411441 & 16.519779 & 0.548652 & 3.598200 \\
2 & 2.919887 & 2.408818 & 2.081370 & 2.163493 \\
3 & 0.123442 & 8.085430 & 0.276777 & 3.432811 \\
4 & 8.932937 & 0.769055 & 1.413136 & 0.100291 \\
5 & 4.889119 & 2.813464 & 1.819007 & 7.213648 \\
6 & 10.608435 & 0.143297 & 11.535804 & 0.494919 \\
7 &0.004825 & 0.010799 & 0.184079 & 0.018554 \\
8 &0.493836 & 11.108500 & 3.756536 & 6.436532 \\
9 &0.647171 & 8.852180 & 6.165669 & 2.576713 \\
10 &5.981172 & 1.493690 & 7.191768 & 1.763514 \\
11 &8.081012 & 13.955062 & 0.332960 & 7.415544 \\
12 &3.645159 & 6.734041 & 6.581146 & 3.503203 \\
13 &9.534242 & 1.280481 & 6.823592 & 3.215876 \\
14 &4.553855 & 5.306134 & 2.320332 & 9.995780 \\
15 &0.338866 & 5.585264 & 1.949093 & 6.221872 \\
\hline
\end{tabular}
\caption{Randomly generated flow scenarios for quadratic cost network flow problem}
\label{table: nf-scenarios}
\end{table}

\newpage

\section{Comparison with Policy Optimization}
\label{sec: policy-opt}
To illustrate the applicability of our performance guarantees, we briefly compare with performance guarantees of policy optimization methods.
Policy optimization methods involve directly learning the policy function $\pi: \mathcal{X} \rightarrow S$, which is a mapping from the feature space to the set of feasible decisions $S$. This is achieved by estimating the policy function using a hypothesis $\pi \in \mathcal{P}$ based on the collected training data set $\{(x_i, \xi_i)\}_{i=1}^n$. That is, applying the ERM principle with policy optimization would lead to minimizing the empirical risk
\begin{equation*}
\min_{\pi \in \mathcal{P}} \hat{R}_n(\pi).
\end{equation*}
One of the advantages of the policy optimization approach is that generalization bounds and finite sample guarantees are apparent from standard statistical learning theory. For example, in Theorem 13 of \cite{bertsimas2020predictive}, a generalization bound is provided with two terms. The first term has a convergence rate of $\mathcal{O}(n^{-\frac{1}{2}})$, while the second term involves the empirical Rademacher complexity of the policy class $\mathcal{P}$. 

Both policy optimization and ICEO suffer from a bias introduced by using a hypothesis class, $\mathcal{P}$ and $\mathcal{H}$, respectively. Indeed, a suitable hypothesis class should be computationally tractable and have a relatively small Rademacher complexity. For example, a linear class would satisfy both of these. Unfortunately, a hypothesis class that satisfies both of these requirements may introduce a bias due to model mis-specification. For policy optimization methods, there may not exist $\pi^* \in \mathcal{P}$ that is ``close enough'' to some function outputting values in the set $W(f^*(x))$). Similarly, for ICEO, there may not exist $f_{\mathcal{H}}^* \in \mathcal{H}$ that fullly characterizes the true underlying hypothesis $f^*$. However, we want to point out that, in practice we generally expect the bias introduced by ICEO to be less than that of policy optimization. It is because ICEO methods model the conditional distribution while the latter model the conditional distribution. Noted that the optimal solution set $W(\cdot)$ may have a complicated structure and require an intricate modeling class $\mathcal{P}$. ICEO does not need to include $w(\cdot)$ as part of the model and only needs to approximate $f^*(\cdot)$ within $\mathcal{H}$, whereas policy optimization necessitates the use of a flexible enough class $\mathcal{P}$ to appropriately model the mapping $w(f^*(\cdot))$. Thus ICEO can be advantageous, especially in situations where the optimization oracle $w(\cdot)$ is complex.

In light of the above discussion, let us solidify our intuition further by comparing the generalization bound of policy optimization with the result provided in Theorem \ref{thm:finite-sample-bd}.
First, note that according to Proposition \ref{prop: oracle-consistency} and Theorem \ref{thm: consistency}, we may take $\rho_n$ as any sequence approaching zero, for example, $\rho_n= \frac{1}{\log(n)}$ suffices. Then according to the result of Corollary \ref{cor: optimal-out-of-sample-bound}, the left-hand side quantifies the error between the out-of-sample risk induced by the policy $w_{\rho_n}\circ \hat{f}^n_{\rho_n}$ relative to the policy $w_{\rho_n}\circ f^\ast_{\mathcal{H}}$ and converges to zero at rate of $\tilde{\mathcal{O}}(n^{-\frac{1}{2}})$. 
For comparison, policy optimization gives a similar bound of error between the out-of-sample risks induced by a learned policy $\hat{\pi}^n$ and the best policy $\pi^\ast$ within the policy class $\mathcal{P}$ with the same convergence rate of $\mathcal{O}(n^{-\frac{1}{2}})$. In both cases, the constant in the convergence bound will be controlled by the Rademacher complexity of the corresponding class, either $\mathcal{H}$ for ICEO or $\mathcal{P}$ for policy optimization. Both methods may introduce some bias, but as we have already argued, we expect it to be easier to achieve smaller bias with smaller complexity for ICEO. In the following, we use a special case (albeit, an extreme one) to illustrate how the flexibility of ICEO can lead to less bias and smaller Rademacher complexity (i.e., faster convergence of the finite-sample performance bound).
\begin{example}
We consider a special case where $\mathcal{X}$ is the simplex $\Delta_K$, which means that the decision-maker has the ability to know the probability vector $p$ directly through the contextual information. In this special case, ICEO has nothing to learn, therefore the hypothesis class should be selected as the singleton identity map $\mathcal{H} =\{f : f(p) = p \ \forall p \in \Delta_K\}$. Then, the multi-variate Rademacher complexity $\mathfrak{R}_n(\mathcal{H}) = 0$, and so the regularization coefficient $\rho$ can be set to zero, i.e., $\rho_n= 0$ for all $n$. On the other hand, policy optimization, which by its presumption is not allowed to use knowledge of the oracle $w(\cdot)$, requires a larger hypothesis class $\mathcal{P}$ to successfully learn the oracle $w(\cdot)$ and so we generally expect $\mathfrak{R}_n(\mathcal{P}) > \mathfrak{R}_n(\mathcal{H}) = 0$.
\end{example}

In addition, considering the fact that, in Corollary \ref{cor: optimal-out-of-sample-bound}, the left-hand side quantifies the error between the out-of-sample risk induced by the policy $w_{\rho_n}\circ \hat{f}^n_{\rho_n}$ relative to the policy $w_{\rho_n}\circ f^\ast_{\mathcal{H}}$. One may notice that $w_{\rho_n}\circ f^\ast_{\mathcal{H}}$ is not the optimal policy $w\circ f^\ast_{\mathcal{H}}$, which introduces bias by setting the decision regularization parameter $\rho$ as a small positive value. However, as demonstrated in Proposition \ref{prop: oracle-consistency}, $w_{\rho_n}\circ f^\ast_{\mathcal{H}}$ converges to the optimal policy $w\circ f^\ast_{\mathcal{H}}$ as $\rho_n$ converges to zero when sample size grows large.

On the other hand, an additional issue that may also limit the applicability of policy optimization is ensuring the feasibility of the output. Unlike in classical machine learning problems like regression or classification, it is difficult to ensure that the hypothesis class $\mathcal{P}$ even guarantees the feasibility of its outputs, whereby $\pi(x) \in S$ for all $\pi \in \mathcal{P}$ and $x \in \mathcal{X}$. This issue does not vanish as sample sizes grow.

Moreover, learning the conditional distribution $f^*(\cdot)$ is arguably more interpretable than policy optimization, as the ICEO approach follows with the estimate-then-optimize structure, and may be easier to accomplish with ``simpler" classes $\mathcal{F}$ since the constraints on the outputs are much simpler (we only need to ensure that $f(x) \in \Delta_K$).

\end{APPENDICES}

\end{document}